%% file: fbst.tex
\documentclass[onecolumn,journal]{IEEEtran}
\usepackage{cite}
\usepackage[cmex10]{mathtools}
\usepackage{amsfonts,amsthm,amssymb, bbm}
\usepackage{thmtools, thm-restate}
\usepackage{dsfont}
\usepackage{mymacros}
\usepackage{color}
\usepackage{algorithm, algorithmicx,algpseudocode}
\usepackage{fixltx2e}
\usepackage{multirow}
\usepackage{datetime}
\usepackage{enumerate}
\usepackage{mathrsfs}
\usepackage{subcaption}
\usepackage[draft]{hyperref}
\usepackage[capitalize]{cleveref}
\crefname{equation}{}{}

\algnewcommand\algorithmicinput{\textbf{INPUT:}}
\algnewcommand\INPUT{\item[\algorithmicinput]}
\algnewcommand\algorithmicoutput{\textbf{OUTPUT:}}
\algnewcommand\OUTPUT{\item[\algorithmicoutput]}

\usepackage{xspace}
\newcommand*{\eg}{\textit{e}.\textit{g}.\@\xspace}
\newcommand*{\ie}{\textit{i}.\textit{e}.\@\xspace}

\newcommand*{\etal}{\textit{et al}.\@\xspace}

\usepackage{graphicx}
\usepackage[outdir=./figs/converted/]{epstopdf}
\DeclareGraphicsExtensions{.eps}

\usepackage{tikz}
\usetikzlibrary{cd}
\usetikzlibrary{dsp}
\usetikzlibrary{decorations.pathreplacing}

\newcommand{\ones}[1]{\mathbf{1}_{#1}}
\newcommand{\zeros}[1]{\mathbf{0}_{#1}}

\renewcommand{\prox}[3]{\mathrm{prox}_{#1 #3}\left(#2\right)}
\newcommand{\Q}[0]{\mathcal{F}}
\newcommand{\barQ}[0]{\bar{F}}
\newcommand{\ww}[0]{\boldsymbol{\omega}}
\newcommand{\Rj}[1]{\mathcal{R}_{#1}}

\declaretheorem[name=Theorem, refname={Theorem,Theorems}]{theorem}
\declaretheorem[name=Lemma, refname={Lemma,Lemmas}]{lemma}
\declaretheorem[name=Proposition, refname={Proposition,Propositions}]{proposition}
\declaretheorem[name=Corollary, refname={Corollary,Corollaries}]{corollary}

\begin{document}

\title{Learning Filter Bank Sparsifying Transforms}

\author{Luke~Pfister,~\IEEEmembership{Student Member,~IEEE,}
        Yoram~Bresler,~\IEEEmembership{Fellow,~IEEE}%
\thanks{
This work was supported in part by the National Science Foundation (NSF)
under Grants CCF 1018660 and CCF-1320953.
}}

\maketitle
\begin{abstract}
  Data is said to follow the transform (or analysis) sparsity model if it becomes sparse when acted
  on by a linear operator called a sparsifying transform. Several algorithms have been designed to
  learn such a transform directly from data, and data-adaptive sparsifying transforms have
  demonstrated excellent performance in signal restoration tasks. Sparsifying transforms are
  typically learned using small sub-regions of data called patches, but these algorithms often
  ignore redundant information shared between neighboring patches.

We show that many existing transform and analysis sparse
representations can be viewed as filter banks, thus linking the local properties
of patch-based model to the global properties of a convolutional model. We
propose a new transform learning framework where the sparsifying transform is
an undecimated perfect reconstruction filter bank. Unlike previous
transform learning algorithms, the filter length can be chosen independently of
the number of filter bank channels. Numerical results indicate filter bank
sparsifying transforms outperform existing patch-based transform learning for
image denoising while benefiting from additional flexibility in the design
process.
\end{abstract}


\section{Introduction}
\label{sec:Introduction}
Countless problems, from statistical inference to geological exploration, can be stated as the
recovery of high-quality data from incomplete and/or corrupted linear measurements. Often, recovery
is possible only if a model of the desired signal is used to regularize the recovery problem.

A powerful example of such a signal model is the \emph{sparse representation}, wherein the signal of
interest admits a representation with few nonzero coefficients. Sparse representations have been
traditionally hand-designed for optimal properties on a mathematical signal class, such as the
coefficient decay properties of a cartoon-like signal under a curvelet representation
\cite{Candes2004}. Unfortunately, these signal classes do not include the complicated and textured
signals common in applications; further, it is difficult to design optimal representations for high
dimensional data. In light of these challenges, methods to \emph{learn} a sparse representation,
either from representative training data or directly from corrupted data, have become attractive.

We focus on a particular type of sparse representation, called 
\emph{transform sparsity}, in which the signal $x \in
\Rbb^N$ satisfies $W x = z + \eta$. The matrix $W \in \Rbb^{K \times N}$ is called a
\emph{sparsifying transform} and earns its name as $z \in \Rbb^K$ is sparse and $\norm{\eta}_2$ is
small \cite{Ravishankar2013b}. Of course, a $W$ that is uniformly zero satisfies this definition but
provides no insight into the transformed signal. Several algorithms have been proposed to
learn a sparsifying transform from data, and each must contend with this type of degenerate solution. The
most common approach is to ensure that $W$ is left invertible, so that $Wx$ is uniformly zero if and
only if $x$ is uniformly zero. Such a matrix is a \emph{frame} for $\Rbb^n$.

In principle, we can learn a sparse representation for any data represented as a vector, including
data from genomic experiments or text documents, yet most research has focused on learning models
for spatio-temporal data such as images. With these signals it is common to learn a model for
smaller, possibly overlapping, blocks of the data called \emph{patches}. We refer to this type of
model as a \emph{patch-based} model, while we call a model learned directly at the image level an
\emph{image-based} model. Patch-based models tend to have fewer parameters than an unstructured
image-based model, leading to lower computational cost and reduced risk of overfitting. In
addition, an image contains many overlapping patches, and thus a model can be learned from a single
noisy image \cite{Elad2006}.

Patch-based models are not without drawbacks. Any patch-based $W$ learned using the usual frame
constraints must have at least as many rows as there are elements in a patch, \ie $W$ must be square
or tall. This limits practical patch sizes as $W$ must be small to benefit from a
patch-based model. 

If our ultimate goal is image reconstruction, we must be mindful of the connection between extracted
patches and the original image. Requiring $W$ to be a frame for patches ignores this relationship
and instead requires that each patch can be independently recovered. Yet, neighboring patches can
be highly correlated- leading us to wonder if the patch-based frame condition is too strict. This
leads to the question at the heart of this paper: \emph{Can we learn a sparsifying transform that
  forms a frame for images, but not for patches, while retaining the computational efficiency of a
  patch-based model? }

In this paper, we show that existing sparsifying transform learning algorithms can be viewed as
learning perfect reconstruction filter banks. This perspective leads to a new approach to learn a
sparsifying transform that forms a frame over the space of images, and is structured as an
undecimated, multidimensional filter bank. We call this structure a \emph{filter bank sparsifying
  transform}. We keep the efficiency of a patch-based model by parameterizing the filter bank in
terms of a small matrix $W$. In contrast to existing transform learning algorithms, our approach can
learn a transform corresponding to a tall, square, or fat $W$. Our learned model outperforms earlier
transform learning algorithms while maintaining low cost of learning the filter bank and the
processing of data by it. Although we restrict our attention to 2D images, our technique is
applicable to any data amenable to patch based methods, such as 3D imaging data.

The rest of the paper is organized as follows.
In Section \ref{sec:preliminaries} we review previous work on transform learning,
analysis learning, and the relationship between patch-based and image-based models.
In Section \ref{sec:pbst_fb} we develop the connection between perfect reconstruction filter banks
and patch-based transform learning algorithms. We propose our filter bank learning algorithm in
Section \ref{sec:learning} and present numerical results in Section \ref{sec:experiments}.  In
Section \ref{sec:remarks} we compare our learning framework to the current crop of deep
learning inspired approaches, and conclude in Section \ref{sec:conclusion}.

\section{Preliminaries}
\label{sec:preliminaries}

\subsection{Notation}
Matrices are written as capital letters, while general linear operators are denoted by script
capital letters such as $\mathcal{A}$. Column vectors are written as lower case letters. The $i$-th
component of a vector $x$ is $x_i$. The $i,j$-th element of a matrix $A$ is $A_{ij}$. We write the
$j$-th column of $A$ as $A_{:,j}$, and $A_{i,:}$ is the column vector corresponding to the transpose
of the $i$-th row. The transpose and Hermitian transpose are $A^T$ and $A^*$, respectively.
Similarly, $\mathcal{A}^*$ is the adjoint of the linear operator $\mathcal{A}$. The $n\times n$
identity matrix is $I_n$. The $n$-dimensional vectors of ones and zeros are written as $\ones{n}$
and $\zeros{n}$, respectively. The convolution of signals $x$ and $y$ is written
$x * y$. For vectors $x, y \in \Rbb^N$, the Euclidean inner
product is $\ip{x, y} = \sum_{i=1}^N x_i y_i$ and the Euclidean norm is written $\norm{x}_2$.
The vector space of matrices $\Rbb^{M \times N}$ is equipped with the inner product $\ip{X, Y} =
\trace{X^T Y}$.
When necessary, we explicitly indicate the vector space on which the inner product is defined; \eg $\ip{x, y}_{\Rbb^N}$.

\subsection{Frames, Patches, and Images}
\label{sub:frames}
A set of vectors $\left\{ \omega_i \right\}_{i=1}^M$ in $\Rbb^m$ is a
\emph{frame} for $\Rbb^m$ if there exists $ 0 < A \leq B < \infty$ such that
\begin{equation}
  \label{eq:frame_condition}
  A \norm{x}_2^2 \leq \sum_{j=1}^M \abs{\ip{x, \omega_i}}^2 \leq B \norm{x}_2^2
\end{equation}
for all $x \in \Rbb^N$ \cite{Christensen2003}. Equivalently, the matrix $\Omega \in \Rbb^{M \times m}$,
with $i$-th row given by $\omega_i$, is left invertible. The frame bounds $A$ and $B$ correspond to the
smallest and largest eigenvalues of $\Omega^T \Omega$, respectively. The frame is \emph{tight} if $A = B$, and
in this case $\Omega^T \Omega = A I_n$. The condition number of the frame is the ratio of the frame bounds,
$B/A$. The $\omega_i$ are called \emph{frame vectors}, and the matrix $\Omega$ implements a
\emph{frame expansion}.

Consider a patch-based model using $K \times K$ (vectorized) patches from an $N \times N$ image.
We call $\Rbb^{K^2}$ the \emph{space of patches} and $\Rbb^{N\times N}$ the \emph{space of images}.
In this setting, transform learning algorithms  find a $W \in \Rbb^{N_c \times K^2}$ with rows that
form a frame for the space of patches\cite{Ravishankar2013b, Ravishankar2013c, Ravishankar2013d, Ravishankar2013a, Wen2014}.

We can extend this $W$ to a frame over the space of images as follows. Suppose the rows of $W$ form
a frame with frame bounds $0 < A \leq B$. Let $\Rj{j} : \Rbb^{N \times N} \to \Rbb^{K^2}$ be the linear
operator that extracts and vectorizes the $j$-th patch from the image, and suppose there are $M$
such patches. So long as each pixel in the image is contained in at least one patch, we have
\begin{equation}
  \label{eq:patch_frame_trivial}
\norm{x}_2^2 \leq \sum_{j=1}^M \norm{\Rj{j} x}_2^2 \leq M \norm{x}_2^2
\end{equation}
for all $x \in \Rbb^{N\times N}$.
Letting $w^i = W_{i, :}$ denote the $i$-th row of $W$, we have for all $x \in \Rbb^{N\times N}$
\begin{align}
n  \sum_{j=1}^M \sum_{i=1}^{N_c} \abs{ \ip{w^i, \Rj{j} x}}^2 &
 \geq \sum_{j=1}^M A \norm{\Rj{j} x}_2^2 \geq A \norm{x}_2^2, \\
\sum_{j=1}^M \sum_{i=1}^{N_c} \abs{ \ip{w^i, \Rj{j} x}}^2 & \leq B \sum_{j=1}^M  \norm{\Rj{j} x}_2^2 \leq M B \norm{x}_2^2.
\end{align}
Because $\ip{w^i, \Rj{j} x}_{\Rbb^{K^2}} = \ip{\Rj{j}^* w^i, x}_{\Rbb^{N \times N}}$, it follows that the collection
$\left\{\Rj{j}^* w^i \right\}_{i=1, j=1}^{N_c, M}$
forms a frame for the space of images with bounds $0 < A \leq MB$.
Thus, every frame over the space of patches corresponds to a frame over the space of images. We
will see that the converse is not true, leading to some of the questions addressed by this paper.

\subsection{Transform Sparsity}
\label{sub:transform}
Recall that a signal $x \in \Rbb^N$ satisfies the transform
sparsity model if there is a matrix $W \in \Rbb^{K \times N}$ such that $Wx = z+\eta$, where $z$
is sparse and $\norm{\eta}_2$ is small.  The matrix $W$ is called a \emph{sparsifying transform} and
the vector $z$ is a \emph{transform sparse code}.
Given a signal $x$ and sparsifying transform $W$, the \emph{transform sparse coding} problem is
\begin{equation}
\label{eq:tsc_p}
\min_z \frac{1}{2}\norm{Wx - z}_2^2 + \nu \psi(z)
\end{equation}
for a sparsity promoting functional $\psi$. Exact $s$-sparsity can be enforced by selecting
$\psi$ to be a barrier function over the set of $s$-sparse vectors. We recognize
\eref{eq:tsc_p} as the evaluation of the proximal operator of $\psi$, defined as
\begin{equation}
\label{eq:prox}
\prox{\psi}{t}{\nu} = \argmin_z \frac{1}{2} \norm{t - z}_2^2 + \nu \psi(z),
\end{equation}
at the point $t=Wx$. Transform sparse coding is often cheap as the proximal mapping of many sparsity
penalties can be computed cheaply and in closed form. For instance, when $\psi(z) = \norm{z}_0$,
then $z = \prox{\psi}{Wx}{\nu}$ is computed by setting $z_i = [Wx]_i$ whenever $\abs{[Wx]_i}^2 >
\nu^2$, and setting $z_i = 0$ otherwise. This operation is called \emph{hard thresholding}.

Several methods have been proposed to learn a sparsifying transform from data, including 
algorithms to learn square transforms \cite{Ravishankar2013b}, orthonormal transforms
\cite{Ravishankar2013c}, structured transforms \cite{Ravishankar2013d}, and overcomplete transforms
consisting of a stack of square transforms \cite{Ravishankar2013a, Wen2014}.
Degenerate solutions are prevented by requiring the rows of the learned transform to constitute a
well-conditioned frame.
In the square case,  the transform learning problem can be written
\begin{equation}
  \label{eq:pbst_obj}
  \min_{W, Z} \frac{1}{2}\norm{W Y - Z}_F^2 + \psi(Z) +  \frac{1}{2} \norm{W}_F^2 - \mu \log{\abs{\det{W}}}
\end{equation}
where $Y$ is a matrix whose columns contain training signals and $\psi$ is a sparsity-promoting
functional. The first term ensures that the transformed data, $WY$, is close to the matrix $Z$,
while the second term ensures that $Z$ is sparse. The remaining terms ensure that $W$ is full rank
and well-conditioned\cite{Ravishankar2013b}. Square sparsifying
transforms have demonstrated excellent performance in image denoising, magnetic resonance imaging,
and computed tomographic reconstruction \cite{Pfister2014, Ravishankar2013, Pfister2014a,
  Pfister2014b}.

\subsection{Analysis sparsity}
\label{sub:analysis}
Closely related to transform sparsity is the \emph{analysis model}. A signal $x \in \Rbb^N$ satisfies the
analysis model if there is a matrix $\Omega \in \Rbb^{K \times N}$, called an analysis operator,  such that $\Omega x = z$ is
sparse.  The analysis model follows by restricting $\eta = \zeros{K}$ in the
transform sparsity model.  

A typical analysis operator learning algorithm is of the form
\begin{equation}
  \label{eq:analysis}
  \min_{\Omega} \psi(\Omega Y)  + J(\Omega)
\end{equation}
where $Y$ are training signals, $\psi$ is a sparsity promoting functional, and $J$ is a regularizer
to ensure the learned $\Omega$ is informative. In the Analysis K-SVD algorithm, the rows of $\Omega$
are constrained to have unit norm, but frame constraints are the most common \cite{Rubinstein2013}.
Yaghoobi \etal observed that learning an analysis operator with $q > K$ rows while using a tight
frame constraint resulted in operators consisting of a full rank matrix appended with $q - K$
uniformly zero rows. They instead proposed a uniformly-normalized tight frame
(UNTF) constraint, wherein the rows of $\Omega$ have equal $\ell_2$ norm and constitute a tight
frame \cite{Yaghoobi2011, Yaghoobi2012, Yaghoobi2013}.

Hawe \etal utilized a similar set of constraints in their GeOmetric Analysis operator Learning
(GOAL) framework \cite{Hawe2013}. They constrained the learned $\Omega$ to the set of full column
rank matrices with unit-norm rows and solved the optimization problem using a manifold descent
algorithm.

Transform and analysis sparsity are closely linked. Indeed, using a variable splitting approach (\eg
$Z = \Omega Y$) to solve \cref{eq:analysis} leads to algorithms that are nearly identical to
transform learning algorithms. The relationships between the transform model, analysis model, and
noisy variations of the analysis model have been explored \cite{Ravishankar2013b}. We focus on the
transform model as the proximal interpretation of sparse coding lends fits nicely within a filter bank
interpretation (see \cref{sec:role-sparsification}).

\subsection{From patch-based to image-based models}
\label{sec:patch_image_models}
A link between patch-based and image-based models can be made using the \emph{Field of Experts}
(FoE) model proposed by Roth and Black \cite{Roth2009}. They modeled the prior probability of an
image as a Markov Random Field (MRF) with overlapping ``cliques'' of pixels that serve as image
patches. Using the so-called Product of Experts framework, a model for the prior probability of an
image patch is expressed as a sparsity-inducing potential function applied to the inner products between
multiple 'filters' and the image patch. A prior for the entire image is formed by taking the
product of the prior for each patch and normalizing.

Continuing in this direction, Chen \emph{et. al} proposed a method to learn an image-based analysis
operator using the FoE framework using a bi-level optimization formulation\cite{Chen2014}.
This approach was recently extended into an iterated filter bank structure called a \emph{Trainable
  Nonlinear Reaction Diffusion} (TNRD) network \cite{Chen2016}.   Each stage of the TNRD network
consists of a set of analysis filters, a channelwise nonlinearity, the adjoint filters, and a direct
feed-forward path. The filters, nonlinearity, and feed-forward mixing weights are trained in a
supervised fashion. The TNRD approach has demonstrated state of the art performance on image
denoising tasks.

The TNRD and FoE algorithms are supervised and use the filter bank structure only as a
computational tool. In contrast, our approach is unsupervised and uses the theory of perfect
reconstruction filter banks to regularize the learning problem.

Cai \etal developed an analysis operator learning method based on a filter bank interpretation of
the operator \cite{Cai2014}. The operator can be thought to act on images, rather than
patches. Their approach is fundamentally identical to learning a square, orthonormal, patch-based
sparsifying transform \cite{Ravishankar2013c}.  In contrast, our approach does
not have these restrictions: we learn a filter bank that is a frame over the space of images,
and corresponds to a tall, fat, or square patch-based transform.

Image-based modeling using the synthesis sparsity model has also been studied \cite{Zeiler2010,
  Papyan2016, Papyan2016a, Wohlberg2016, Garcia-Cardona2017}.
We briefly discuss the
convolutional dictionary learning (CDL) problem; for details, see the recent reviews
\cite{Wohlberg2016, Garcia-Cardona2017}.

The goal of CDL is to find a set of filters, $\left\{ d_i \right\}_{i=1}^{N_c}$, such that the
training signals can be modeled as $y = \sum_{i=1}^{N_c} d_i * a_i$, where the $a_i$ are sparse.
Here, $y$ is an image, not a patch. The filters $d_i$ are required to have compact support so as to
limit the number of free parameters in the learning problem. The desired convolutional structure can
be imposed by writing the convolution in the frequency domain, but care must be taken to ensure that
the $d_i$ remain compactly supported.  

In the next section, we show that patch-based analysis and transform models, in contrast to
synthesis models, are naturally endowed with a convolutional structure.

\section{Patch-based Sparsifying Transforms as Filter Banks}
\label{sec:pbst_fb}
In this section, we illustrate the connections between patch-based sparsifying transforms and
multirate finite impulse response (FIR) filter banks. The link between patch-based analysis methods
and convolution has been previously established, but thus far has been used only as a computational tool and the
effect of stride has not been considered \cite{Peyre2011, Cai2014, Roth2009, Chen2014, Chen2012}.
Our goal is to illustrate how and when the boundary conditions, patch stride, and a patch-based
sparsifying transform combine to form a frame over the space of images.

Let $x \in \Rbb^{N \times N}$ be an image, and let $W \in \Rbb^{N_c \times K^2}$ be a sparsifying
transform intended to act on (vectorized) $K \times K$ patches of $x$. We consider two ways to
represent applying the transform to the image.

The usual approach is to form the patch matrix $X \in \Rbb^{K^2 \times M^2}$ by vectorizing $K
\times K$ patches extracted from the image $x$, as illustrated in \cref{fig:patch_mat}. We call the
spacing between adjacent extracted patches the \emph{stride} and denote it by $s$. The extracted
patches overlap when $s < K$ and are disjoint otherwise. We assume the stride is the same in both
horizontal and vertical directions and evenly divides $N$. The number of patches, $M^2$, depends on
the boundary conditions and patch stride; \eg $M^2 = N^2/s^2$ if periodic boundary conditions are
used.  The patch matrix for the transformed image is $W X \in \Rbb^{N_c \times M^2}$. 

Our second approach eliminates patches and their vectorization by viewing $WX$
as the output of a multirate filter bank with 2D FIR filters and input $x$.  

Let $\mathcal{H} : \Rbb^{N \times N} \to \Rbb^{N_c} \otimes \Rbb^{M \times M}$ be this filter bank 
operator, which transforms an $N \times N$ image to a third-order tensor formed as a stack of $N_c$
output images, each of size $M \times M$.

We build $\mathcal{H}$ from a collection of downsampled convolution operators. For $i=1, 2, \hdots
N_c$, we define the $i$-th channel operator $ \mathcal{H}_i : \Rbb^{N \times N} \to \Rbb^{M \times
  M}$ such that $[\mathcal{H}_i x]_{a,b} = [h_i * x]_{s a, s b}$. The stride $s$ dictates the
downsampling level, and the patch extraction boundary conditions determine the convolution boundary
conditions; in particular, if periodic boundary conditions are used, then $\mathcal{H}_i$ implements
cyclic convolution. The impulse response $h_i$ is obtained from the $i$-th row of $W$ as $\Rj{1}^*
w^i$. This matrix consists of a $K \times K$ submatrix embedded into the upper-left corner of an $N
\times N$ matrix of zeros as illustrated in Figure \ref{fig:filter_from_patch}. \footnote{In the case of
cyclic convolution, $\Rj{1}^* w^i$ is exactly the impulse response of the $i$-th channel, but only
the nonzero portion of $\Rj{1}^* w^i$ is the impulse response when using linear convolution. In a
slight abuse of terminology, we call $\Rj{1}^* w^i$ the impulse response in both instances.}

Finally, we construct $\mathcal{H}$ by ``stacking'' the channel operators: $\mathcal{H} =
\sum_{i=1}^{N_c} e_i \otimes \mathcal{H}_i$, where $e_i$ is the $i$-th standard basis vector in
$\Rbb^{N_c}$ and $\otimes$ denotes the Kronecker (or tensor) product. With this definition, $ y=
\mathcal{H}x = \sum e_i \otimes \mathcal{H}_i x = \sum e_i \otimes y_i$. The filter bank structure
is illustrated in \cref{fig:analysis_fb}. We refer to $\mathcal{H}$ constructed in this form as a
\emph{filter bank sparsifying transform}. The following proposition links $WX$ and $\mathcal{H} x$:

\begin{proposition}
  \label{prop:filter_bank}
  Let $X \in \Rbb^{K^2 \times M^2}$ be a patch matrix for image $X$, and let $W \in \Rbb^{N_c \times
    K^2}$. The rows of $WX$ can obtained by passing $x$ through the $N_c$ channel, 2D FIR multirate
  analysis filter bank $\mathcal{H}$ and vectorizing the channel outputs.
\end{proposition}
A proof for 1D signals is given in Appendix \ref{app:conv_proof}.  The proof for 2D is similar,
using vector indices.

Proposition \ref{prop:filter_bank} connects the local, patch-extraction process and the matrix $W$
to a filter bank operator that acts on images. Unlike convolutional synthesis models, patch-based
analysis operators naturally have a convolutional structure.

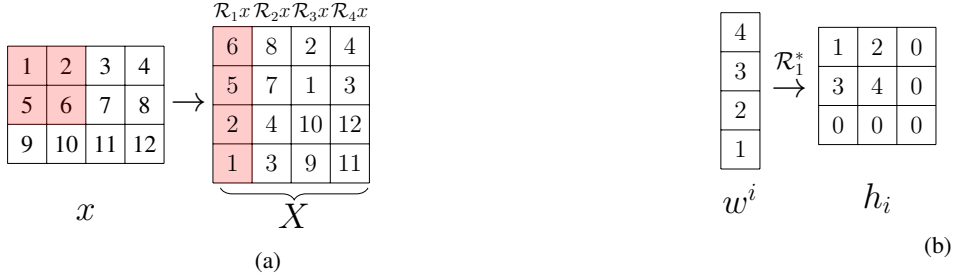
\begin{figure}[t]
  \centering
  \begin{subfigure}{0.4\columnwidth}
  \scalebox{0.65}{\input{figs/patch_fig.tex}}
  \caption{}
  \label{fig:patch_mat}
\end{subfigure}
\hspace{4.0em}
  \begin{subfigure}{0.4\columnwidth}
  \scalebox{0.65}{\hspace{30pt}\input{figs/filter_from_patch.tex}}
  \caption{}
  \label{fig:filter_from_patch}
\end{subfigure}
\caption{(a) Construction of the patch matrix $X\in\Rbb^{4 \times 4}$ from $2 \times 2$ patches of $x
    \in \Rbb^{3 \times 4}$ using periodic boundary conditions and a stride of $2$.
    Note that the vectorized patch is ``flipped'' from the natural ordering; \ie, the top-left pixel
    in the patch is the final element of the vector.
    (b) Obtaining the impulse response $h_i$ from
    the $i$-th row of the patch based sparsifying transform $w^i  = W_{i,:}$.
}
\end{figure}

\begin{figure}[t]
\begin{center}
  \input{./figs/filterbank_analysis_only.tex}
\end{center}
\caption{Analysis filter bank generated by a sparsifying transform $W$ and stride length $s$.
}
\label{fig:analysis_fb}
\end{figure}
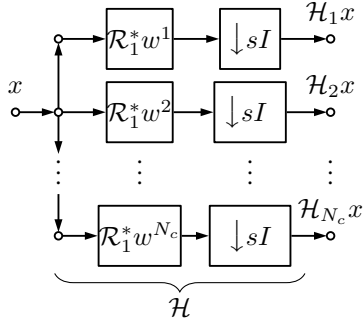


Next, we investigate connections between the frame properties of $\mathcal{H}$ and the combination
of $W$ and the patch extraction scheme. Our primary tool is the polyphase representation of filter
banks \cite{Vaidyanathan1992, Do2011a}. Consider the image $x$ as a 2D sequence $x[n_1, n_2]$ for $0
\leq n_1, n_2 \leq N-1$. The $z$-transform of the $(a, b)$-th polyphase component of $x$ is
$z$-transform $\hat{X}_{a,b}(\mathbf{z}) = \sum_{n_1, n_2} x[n_1\cdot s + a, n_2 \cdot s + b]
z_1^{-n_1} z_2^{-n_2}$ of the shifted and downsampled sequence, where $\mathbf{z} = [z_1, z_2]\in
\Cbb^2$ and $0 \leq a, b \leq s-1$. The polyphase representation for the sequence $x$ is formed by
stacking the polyphase components in lexicographical order into a single $\hat{X}(\mathbf{z}) =
\left[X_{0,0}(\mathbf{z}), \hdots, X_{s-1, s-1}(\mathbf{z})\right]^T \in \Cbb^{s^2}$.

The filter bank $\mathcal{H}$ has a polyphase matrix $\hat{H}(\mathbf{z}) \in \Cbb^{Nc
  \times s^2}$ formed by stacking the polyphase representations of each channel into a row, and
stacking the $N_c$ rows. Explicitly,
\begin{equation}
  \label{eq:polyphase_mat}
  \hat{H}(\mathbf{z}) = \begin{bmatrix}
    \hat{H}^{0}_{0,0}(\mathbf{z}) &  \hat{H}^{0}_{0, 1}(\mathbf{z}) & \hdots & \hat{H}^{0}_{s, s}(\mathbf{z}) \\
    \hat{H}^{1}_{0,0}(\mathbf{z}) &  \hat{H}^{1}_{0, 1}(\mathbf{z}) & \hdots & \hat{H}^{1}_{s, s}(\mathbf{z}) \\
    \vdots & \vdots & \ddots & \vdots \\
    \hat{H}^{N_c-1}_{0,0}(\mathbf{z}) &  \hat{H}^{N_c-1}_{0,1}(\mathbf{z}) & \hdots & \hat{H}^{N_c-1}_{s, s}(\mathbf{z}) \\
  \end{bmatrix}
\end{equation}
where $\hat{H}^{i}_{a,b} (\mathbf{z})$ is the $(a,b)$-th polyphase component of the $i$-th filter in
$\mathcal{H}$. The entries of $\hat{H}(\mathbf{z})$ are, in general, bi-variate polynomials in
$\mathbf{z} = [z_1, z_2]$. The output of the filter bank, $y= \mathcal{H} x$, can be written in the polyphase
domain as $\hat{Y}(\mathbf{z}) = \hat{H}(\mathbf{z}) \hat{X}(\mathbf{z})$, where the $i$-th element
of the vector $\hat{Y}(\mathbf{z})$  is the $z$-transform of the $i$-th output channel.

Many important properties of $\mathcal{H}$ are tied to its polyphase matrix. An analysis filter
bank $\mathcal{H}$ is said to be \emph{perfect reconstruction} (PR) if there is a (synthesis) filter
bank $\mathcal{G}$ such that $\mathcal{G} \mathcal{H} = \mathcal{I}$, or in the polyphase domain
$\hat{G}(\mathbf{z}) \hat{H}(\mathbf{z}) = I$. A filter bank is PR if and only if
$\hat{H}(\mathbf{z})$ is full column rank on the unit circle \cite{Cvetkovic1998}. A filter bank is
said to be \emph{orthonormal} if $\mathcal{H}^* \mathcal{H} = \mathcal{I}$, that is, the filter bank
with analysis section $\mathcal{H}$ and synthesis section equal to the adjoint of $\mathcal{H}$ is
an identity mapping on $\Rbb^{N \times N}$. In the polyphase domain, this corresponds to
$\hat{H}^*(\mathbf{z}^{-1})\hat{H}(\mathbf{z}) = I$, where the star superscript denotes Hermitian
transpose and $\mathbf{z}^{-1} = [z^{-1}_1, z^{-1}_2]$ \cite{Strang1996, Vaidyanathan1992}.

A PR filter bank implements a frame expansion over the space of images, and an orthonormal filter
bank implements a tight frame expansion over the same space \cite{Bolcskei1998,Martin1995}. The
frame vectors are the collection of the shifts of the impulse responses of each channel, and are
precisely the collection $\left\{\Rj{j}^* w^i\right\}$ discussed in Section \ref{sub:frames}.
The link between patch-based transforms and filter banks does not directly lead to new transform
learning algorithms, as the characterization and construction of multidimensional PR filter
banks is hard due to the lack of a multidimensional spectral factorization theorem
\cite{Venkataraman1994, Do2011a, Zhou2005, Delgosha2004}.  

Next, we study we illustrate the connections between patch-based sparsifying transforms and perfect
reconstruction filter banks as a function of the stride length.  We show that in certain cases the
PR condition takes on a simple form.

\subsection{Perfect Recovery: Non-overlapping patches}
\label{sec:non_overlapping}
Consider $s = K$, so that the extracted patches do not overlap. Applying the sparsifying transform
$W$ to non-overlapping patches is an instance of a \emph{block transform} \cite{Malvar1992}. Block
transforms are found throughout in signal processing applications; for example, the JPEG compression
algorithm. Block transforms are viewed as a decimated FIR filter bank with uniform downsampling by
$K$ in each dimension, consistent with Proposition \ref{prop:filter_bank}.

It is informative to view patch-based transform learning algorithms through the lens of block
transformations. Because we downsample by $K$ in each dimension, and the filters are of size $K
\times K$, the polyphase matrix $\hat{H}(\mathbf{z})$ is constant in (independent of) $\mathbf{z}$
and is equal to $W$. This gives a direct connection between the PR properties of $\mathcal{H}$,
which acts on images, and $W$, which acts on patches. Patch-based transform learning algorithms
enforce either invertibility of $W$ (in the square case) or invertibility of $W^T W$ (in the
overcomplete case), and thus $\mathcal{H}$ is PR. If $W$ is orthonormal, so too is $\mathcal{H}$.

\subsection{Perfect Recovery: Partially overlapping patches}
Next, consider patches extracted using a stride $1 < s < K$. While $WX$ is no longer a block
transformation, it is related to a \emph{lapped transformation} \cite{Malvar1992}. Lapped transforms
aim to reduce artifacts that arise from processing each block (patch) independently by allowing for
overlap between neighboring blocks. Many lapped transforms, such as the Lapped Orthogonal Transform,
the Extended Lapped Transform, and the Generalized Lapped Orthogonal Transform
\cite{Poularikas2009}, enjoy both the PR property and efficient implementation.

Lapped transforms were designed for signal coding applications. The number of channels in
the filter bank is decreased as the degree of overlap increases, so that the number of transform
coefficients using a lapped transform is the same as using a non-lapped transform. While redundancy
may be undesirable in certain coding applications, it aids the solution of inverse problems by allowing for
richer and more robust signal models \cite{Goyal1998}. We allow the stride length to decrease while
keeping the number of channels fixed, and interpret $WX$ as a ``generalized'' lapped transform.
When the stride is less than $K$, $W$ no longer corresponds to the polyphase matrix of the filter
bank $\mathcal{H}$; instead, the polyphase matrix $\hat{H}(\mathbf{z})$ will contain high-order, 2D
polynomials. While the filter bank may still be PR, the PR property is not directly implied by
invertibility of $W$.

We can learn a PR generalized lapped transform by enforcing the more restrictive PR conditions for
non-overlapping patches, that is, invertibilty of $W^T W$. When $s = 1$, this technique is
equivalent to cycle spinning, which was developed to add shift-invariance to decimated wavelet
transforms \cite{Coifman1995}. When $1 < s < K$, we can interpret $\mathcal{H} x$ as cycle spinning
without all possible shifts.

\subsection{Perfect Recovery: Maximally overlapping patches}
Finally, consider extracting maximally overlapping patches by setting $s = 1$. The resulting
filter bank $\mathcal{H}$ is undecimated and the Gram operator $\mathcal{H}^* \mathcal{H}$ is
shift invariant. As there is no downsampling, the polyphase representations of $x$ and $y$ are the
$z$-transforms of the sequences $x$ and $y$. 
The polyphase matrix of $\mathcal{H}$ is the column vector
$\hat{H}(\mathbf{z}) = [\hat{H}_1(\mathbf{z}), \hdots \hat{H}_{N_c}(\mathbf{z})]^T$ where
$\hat{H}_i(\mathbf{z})$ is the $z$-transform of $h_i = \Rj{1}^*w^i$.

An undecimated linear convolution filter bank is PR if and only if its filters have no common zeros
on the unit circle \ie, each frequency must pass through at least one channel of the filter bank
\cite{Cvetkovic1998}. When evaluated on the unit circle the $z$-transform becomes the Discrete Time
Fourier Transform (DTFT), defined for $h \in \Rbb^{K \times K}$ as
\begin{equation}
  \label{eq:dtft}
  \hat{H}(\ww) = \sum_{n_1=0}^{K-1}\sum_{n_2=0}^{K-1} h[n_1, n_2] e^{-j \omega_1 n_1}e^{-j \omega_2 n_2}
\end{equation}
where $\ww = [\omega_1, \omega_2]$ with $\omega_1, \omega_2 \in [0, 2 \pi)$. Now, the polyphase matrix is full rank on the unit circle if and only if 
\begin{equation}
  \label{eq:phi_pr}
  \varphi(\ww) \triangleq \sum_{i=1}^{N_c} \abs{\hat{H}_i(\ww)}^2 > 0 \quad \forall w_1, w_2 \in [0, 2\pi),
\end{equation}
where $\varphi(\ww)$ is the DTFT of the impulse response of $\mathcal{H}^*
\mathcal{H}$ and is an even, real, non-negative, 2D trigonometric polynomial with maximum component
order $K - 1$. Explicitly, 
\begin{equation}
\varphi(\ww) = \sum_{n_1={-K+1}}^{K-1} \sum_{n_2=-K+1}^{K-1} \tilde{h}[n_1, n_2] \cos(\omega_1 n_1) \cos(\omega_2 n_2),
\end{equation}
where the impulse response of $\tilde{h}$ is $\mathcal{H}^* \mathcal{H}$ is the 
sum of the channel-wise autocorrelations; that is, 
\begin{equation}
\tilde{h}[n_1, n_2] = \sum_{i=1}^{N_c}\sum_{l_1=-\infty}^\infty\sum_{l_2=-\infty}^{\infty} h_i[l_1, l_2] h_i[l_1 - n_1, l_2 - n_2].
\end{equation}

Direct verification of the PR condition \eqref{eq:phi_pr} is NP-Hard for $K \geq 2$, underlining 
the difficulty of multidimensional filter bank design \cite{Murty1987, Parrilo2003}. We sidestep the
difficulty of working with \eqref{eq:phi_pr} by developing the PR condition when image patches are
extracted using periodic boundary conditions. The resulting filter bank implements \emph{cyclic}
convolution. Afterwards, we show that under certain conditions, the PR property of a cyclic
convolution filter bank implies the PR property of a linear convolution filter bank constructed from
the same filters.

\subsubsection{Periodic Boundary Conditions / Cyclic Convolution}
If image patches are extracted using periodic boundary conditions, the channel operators
$\mathcal{H}_i : \Rbb^{N\times N} \to \Rbb^{N \times N}$ implement cyclic convolution and are 
diagonalized by the 2D Discrete Fourier Transform (DFT). Let $\Q$ be the orthonormal 2D-DFT operator such
that
\begin{equation}
  \label{eq:dft}
  (\Q h_i)[\mathbf{k}] = N^{-1} \sum_{n_1=0}^{K-1}\sum_{n_2=0}^{K-1} h_i[n_1, n_2] e^{-j \frac{2 \pi k_1 n_1}{N}}e^{-j \frac{2 \pi k_2 n_2}{N}}
\end{equation}
for $\mathbf{k} = [k_1, k_2]$ and $0 \leq k_1, k_2 < N$;  that is, the length $N$ 2D-DFT of the filter
$h_i$ padded with $N - K$ zeros in each dimension. Define $\mathcal{D}_i \in \Cbb^{N \times N}
\to \Cbb^{N \times N}$ as the operator that multiplies pointwise by $\Q h_i$:  for $u \in
\Cbb^{N\times N}$, we have $(\mathcal{D}_i u)(\mathbf{k}) = (\Q h_i)(\mathbf{k}) \cdot
u(\mathbf{k})$. The cyclic convolution operator $\mathcal{H}_i$ has eigenvalue decomposition
$\Q^* \mathcal{D}_i \Q$.  We can use this channel-wise decomposition to find the spectrum of
$\mathcal{H}^* \mathcal{H}$:
\begin{lemma}
  \label{lem:spectrum}
  The $N^2$ eigenvalues of the undecimated cyclic analysis-synthesis filter bank $\mathcal{H}^* \mathcal{H}$ are
  given by $\sum_{i=1}^{N_c} \abs{(\Q h_i)[\mathbf{k}]}^2$ for $\mathbf{k} = [k_1, k_2]$
  and $0 \leq k_1, k_2 < N$.
\end{lemma}
\begin{proof}
  We have
  \begin{align*}
    \mathcal{H}^* \mathcal{H} &= \sum_{i=1}^{N_c} (e_i \otimes \mathcal{H}_i)^*
    (e_i \otimes \mathcal{H}_i) = \sum_{i=1}^{N_c} \mathcal{H}_i^* \mathcal{H}_i \\
                              &= \Q^* \sum_{i=1}^{N_c} \mathcal{D}_i^* \mathcal{D}_i \Q = \Q^* \mathcal{D} \Q
  \end{align*}
  where $(\mathcal{D} u)[\mathbf{k}] = \sum_{i=1}^{N_c} \abs{(\Q h_i)[\mathbf{k}]}^2 \cdot u[\mathbf{k}]$.
  \end{proof}
  The quantity $\abs{ (\Q h_i)[\mathbf{k}}]^2$ is the squared magnitude response of the $i$-th
  filter evaluated at the DFT frequency $\mathbf{k}$, and the eigenvalues of $\mathcal{H}^*
  \mathcal{H}$ are the sum over the $N_c$ channels of these squared magnitude responses. As the DFT
  consists of samples of the DTFT, by Lemma \ref{lem:spectrum} and \eqref{eq:phi_pr}, the
  eigenvalues of $\mathcal{H}^* \mathcal{H}$ can be seen to be samples of the trigonometric
  polynomial $\varphi(\ww)$ over the set $\Theta_N = \left\{ \left(\frac{2 \pi k_1}{N}, \frac{2 \pi
        k_2}{N}\right) : 0\leq k_1, k_2 < N \right\}$.

Recall that $\mathcal{H}$ implements a frame expansion only if the smallest eigenvalue of
$\mathcal{H}^* \mathcal{H}$ is strictly positive \cite{Christensen2003}.  We have
the following PR condition for cyclic convolution filter banks:
\begin{corollary}
  \label{cor:pr_condition}
  The undecimated cyclic filter bank $\mathcal{H}$ implements a frame expansion for $\Rbb^{N \times
    N}$ if and only if $\sum_{i=1}^{N_c} \abs{ (\Q h_i)[\mathbf{k}]}^2 > 0$ for $0 \leq k_1, k_2 <
  N$.  If $\mathcal{H}$ implements a frame expansion, the upper and lower frame bounds  are
  $\min_\mathbf{k} \sum_{i=1}^{N_c} \abs{ (\Q h_i)[\mathbf{k}]}^2$
  and
  $\max_\mathbf{k} \sum_{i=1}^{N_c} \abs{ (\Q h_i)[\mathbf{k}]}^2$.
\end{corollary}

Whereas the PR condition for a linear convolution filter bank must hold over the
unit circle, the PR condition for cyclic convolution filter bank involves only the $N^2$ DFT
frequencies.

The factorization $\mathcal{H}^* \mathcal{H} = \mathcal{F}^* \mathcal{D} \mathcal{F}$ also provides an
easy way to compute the (minimum norm) synthesis filter bank
$\mathcal{H^\dagger}$ that satisfies $\mathcal{H}^\dagger \mathcal{H} = \mathcal{I}$.
We have $\mathcal{H}^\dagger = (\mathcal{H}^* \mathcal{H})^{-1} \mathcal{H}^*$, and the necessary
inverse is given by $(\mathcal{H}^* \mathcal{H})^{-1} = \Q^* \mathcal{D}^{-1} \Q$.

\subsubsection{Return to Linear Convolution}
We now want to link the PR conditions for cyclic and linear convolution filter banks. Recently, we
have shown \cite{Pfister2018} that the minimum value of a real, multivariate trigonometric
polynomial can be lower bounded given sufficiently many uniformly spaced samples of the polynomial,
provided that the polynomial does not vary too much over the sampling points.

\begin{theorem}
  \label{thm:poly_nn_condition}
  \cite{Pfister2018}
  Let $\varphi(\ww)$ be a real, non-negative, two-dimensional trigonometric polynomial with maximum
  component order $n$.
  Define $\Theta_N = 
  \left\{ \left(\frac{2 \pi k_1}{N}, \frac{2 \pi
        k_2}{N}\right) : 0\leq k_1, k_2 < N \right\}$ where $N \geq 2n + 1$. If 
  $\kappa_N \triangleq \frac{\max_{\ww \in \Theta_N} \varphi(\ww)}
  {\min_{\ww \in \Theta_N} \varphi(\ww)}$ satisfies
  \begin{equation}
    \kappa_N \leq  \frac{N}{n} - 1,
  \end{equation}
  then $\varphi(\ww) > 0$ for all $\omega_1, \omega_2 \in [0, 2\pi)$.
\end{theorem}
If $\varphi(\ww)$ is defined by \eqref{eq:phi_pr}, then
$\kappa_N$ is the frame condition number of
a cyclic convolution filter bank operating on $N \times N$ images.  
Theorem \ref{thm:poly_nn_condition} is the link between PR properties of
cyclic and linear convolution filter banks we desired, and we have the following PR condition for linear
convolution filter banks:
\begin{corollary}
  \label{cor:linear_pr_condition}
  Let $\mathcal{H}_C$ be an undecimated cyclic convolution filter bank with $K \times K$ filters
  that operates on $N \times N$ images, with frame condition number $\kappa_N$.  Let $\mathcal{H}$
  be a linear convolution filter bank constructed from the same filters as $\mathcal{H}_C$.  Then
  $\mathcal{H}$ is PR if $\kappa_N \leq \frac{N}{K-1} - 1$.
\end{corollary} 
\begin{proof}
  Take $n = K - 1$ in Theorem \ref{thm:poly_nn_condition}.
\end{proof} 
Corollary \ref{cor:linear_pr_condition} states that well-conditioned PR cyclic convolution filter banks,
with filters that are short relative to image size $N$,  are also PR \emph{linear} convolution filter banks.

The PR conditions of Corollaries \ref{cor:pr_condition} and \ref{cor:linear_pr_condition} are
significantly more general than the patch based PR conditions. For example, $W \in \Rbb^{N_c \times
  K^2}$ can be left-invertible only if $N_c \geq K^2$. The PR conditions of Corollaries
\ref{cor:pr_condition} and \ref{cor:linear_pr_condition} have no such requirements; indeed, a single
channel ``filter bank'' can be PR. Our PR conditions are easy to check, requiring only
the 2D DFT of $N_c$ small filters.
\subsection{The role of sparsification}
\label{sec:role-sparsification}
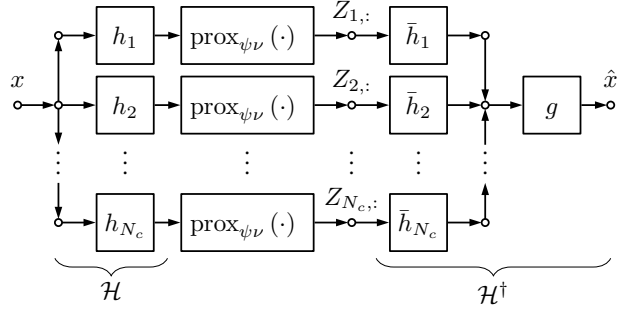
\begin{figure}[t]
\begin{center}
  \scalebox{0.95}{\input{./figs/filterbank_fig.tex}}
\end{center}
\caption{Analysis-synthesis filter bank generated by sparsifying transform $W$ and separable
  sparsity penalty $\psi$.
  Here, $h_i = \Rj{1}^* w^{i}$,
  the impulse response $\bar{h}_i$ is the flipped version of $h_i$, and
  $g$ is the impulse response for the filter $(\mathcal{H}^* \mathcal{H})^{-1}$.
}
\label{fig:filterbank_block_diagram}
\end{figure}

We have interpreted the transformed image patches $WX$ as the output of a filter bank. The sparse
matrix $Z$ in~\eqref{eq:tsc_p} can be viewed as passing the filter bank output through a nonlinear
function implementing $\prox{\psi}{WX}{\nu}$. This interpretation is particularly appealing whenever
$\psi$ is coordinate-wise separable, meaning $\psi(z) = \sum_i \psi(z_i)$. Then the transform sparse
code for the $j$-th channel depends only on the $j$-th filtered channel and is given by
$\prox{\psi}{ \mathcal{H}_j x}{\nu}$. The resulting nonlinear analysis-synthesis filter bank is
illustrated in Figure \ref{fig:filterbank_block_diagram}. If the input signal $x$ is indeed
perfectly sparsifiable by the filter bank (\ie, $\mathcal{H}x = \prox{\psi}{\mathcal{H}x}{\nu}$),
then the output of the analysis stage is invariant to the application of the nonlinearity and
the entire system functions as an identity operator.

We can replace the usual soft or hard thresholding functions by exotic nonlinearities, such as the
firm thresholding function \cite{Gao1997} or generalized $p$-shrinkage functions
\cite{Chartrand2014}. These nonlinearities have led to marginally improved performance in image
denoising \cite{Chen2014} and compressed sensing \cite{Chartrand2009}. Alternatively, we can abandon
the interpretation of the nonlinearity as a proximal operator and instead learn a custom
nonlinearity for each channel, either in a supervised setting \cite{Schmidt2014, Chen2016} or in an
unsupervised setting with a Gaussian noise model \cite{Pfister2017}.

Filter bank sparsifying transforms share many similarities with convolutional autoencoder (CAE)
\cite{Masci2011}. Both consist of a filter bank followed by a channelwise nonlinearity. However, in
the case of a CAE, the ``decoder'' $\mathcal{H}^\dagger$ is typically discarded and the
output of the ``encoder'', $\prox{\psi}{\mathcal{H} x}{\nu}$, is passed into additional layers for
further processing.

\subsection{Principal Component Filter Banks}
The previous sections have shown that transform learning can be viewed as adapting a filter bank to
sparsify our data. A similar problem is the design of \emph{principal component filter banks}
(PCFB). Let $\mathscr{C}$ denote a set of orthonormal filter banks, such as $M$-channel filter banks
with a particular downsampling factor, and let $x$ be a given input signal. A filter bank
$\mathcal{H}^{\mathscr{P}}$ is said to be a PCFB for the class $\mathscr{C}$ and the input $x$ if,
for any $\mathcal{H} \in \mathscr{C}$ and all $m=1, \hdots, M$,
\begin{equation}
  \label{eq:majorization}
\sum_{i=1}^m a^2_ i \geq \sum_{i=1}^m b^2_i
\end{equation}
where $a_i$ and $b_i$ are the $\ell_2$ norms of the $i$-th channel of $\mathcal{H}^{\mathscr{P}} x$
and $\mathcal{H} x$, respectively \cite{Tsatsanis1995}. A PCFB provides compaction of the output
energy into lower-indexed channels, and thus minimal $\ell_2$ error when reconstructing a signal
from $m < M$ filtered components. The existence and design of PCFBs in 1D is well studied
\cite{Moulin1998}. However, the design of multidimensional FIR PCFBs is
again made difficult due to the lack of a multidimensional spectral factorization theorem, although
suboptimal algorithms exist \cite{Xuan1995, Xuan1998}.

There are superficial similarities between the design of PCFBs and transform learning, especially when
$W$ is restricted to be square and orthonormal. The sparsity of the transformed signal implies a
form of energy compaction. However, we impose no constraints on location of non-zero coefficients
and thus the learned transform is unlikely to satisfy the majorization property
\eqref{eq:majorization}. Further, an orthogonal $W$ matrix induces an orthogonal filter bank only if
non-overlapping patches are used. The PCFB for such a block transformation is known to be the
Karhunen-Loeve transformation of the data \cite{Unser1993}, from which the learned transform
can differ substantially \cite{Ravishankar2013b}. Conversely, the energy majorization property
\eqref{eq:majorization} does not imply sparsity of the channel outputs, and a PCFB will not, in
general, be a filter bank sparsifying transform.

\section{Learning a sparsifying filter bank}
\label{sec:learning}

We briefly review methods to incorporate an adaptive sparsity model to the solution of the inverse
problems. We consider two paradigms: in the ``universal'' paradigm, our sparsifying
transform $\mathcal{H}$ is learned off-line over a set of training data. In the ``adaptive''
paradigm, the transform is learned during the solution of the inverse problem, typically by
alternating between a signal recovery phase and a transform update phase.
For synthesis dictionary
learning it has been reported that the adaptive method typically works better for lower noise levels
while the universal method shines in high noise \cite{Elad2006}.
In both paradigms, we learn sparsifying transform by minimizing a function that is designed to
encourage ``good'' transforms.  

\subsection{Problem Formulation}
We now develop a method to learn an undecimated filter bank sparsifying transform that
takes advantage of the flexibility granted by the PR conditions of
\cref{cor:pr_condition,cor:linear_pr_condition}.
Let $x$ be a training signal, possibly drawn from a set of training signals.
We wish to learn a filter bank sparsifying transform that satisfies four properties:
\begin{enumerate}[(D1)]
  \item $\mathcal{H}x$ should be sparse; \label{d:sparse}
  \item $\mathcal{H}$ should be left invertible and well conditioned; \label{d:frame}
  \item $\mathcal{H}$ should contain no duplicated or uniformly zero filters;\label{d:coherence}
  \item $\mathcal{H}$ should be have few degrees of freedom\label{d:dof}.
\end{enumerate}
Properties (D\ref{d:sparse}) -- (D\ref{d:coherence}) ensure our transform is ``good'' in 
that it sparsifies the training data, is a well-behaved frame over the space of images,
and is not overly
redundant. Property (D\ref{d:dof}) ensures good generalization properties under the universal
paradigm and prevents overfitting under the adaptive paradigm.

As with previous transform learning approaches, we satisfy (D\ref{d:sparse}) by minimizing
$\frac{1}{2}\norm{\mathcal{H}x - z}^2 + \nu \psi(z)$ where $\psi$ is a sparsity-promoting
functional. The first of term is called the \emph{sparsification error}, as it measures the distance
between $\mathcal{H}x$ and its sparsified version, $z$.

We satisfy (D\ref{d:dof}) by writing the action of the sparsifying transform on the image as $W X$,
where $W \in \Rbb^{N_c \times K^2}$ and where $X$ is formed by extracting and vectorizing $K \times
K$ patches with unit stride. This parameterization ensures that the learned filters are compactly
supported and have only $N_c K^2$ free parameters. We emphasize that $WX$ and $\mathcal{H} x$ are
equivalent modulo a reshaping operation. Both expressions should be thought of independently of the
computational tool used to calculate the results; $WX$ can be implemented using Fourier-based fast
convolution algorithms, just as $\mathcal{H} x$ can be implemented by dense matrix-matrix
multiplication.

The sparsification error is written as
\begin{equation}
\label{eq:se}
f(W, Z, x) = \frac{1}{2} \norm{W X - Z}_F^2,
\end{equation}
where the $j$-th row of $Z \in \Rbb^{N_c \times N^2 }$ is the sparse code for the $j$-th
channel output. We can learn a transform over several images by 
summing the sparsification error for each image.

We promote transforms that satisfy (D\ref{d:frame}) through the penalty $\frac{1}{2}
\sum_{j=1}^{N_c}\norm{W_{j,:}}_2^2 - \log\det \mathcal{H}^* \mathcal{H}$. The log determinant term
ensures that no eigenvalues of $\mathcal{H}^* \mathcal{H}$ become zero, while the $\ell_2$ norm term
ensures that the filters do not grow too large. This penalty can be written as $\sum_{j=1}^{N^2}
\lambda_i - \log{\lambda_i}$, where the $\lambda_i$ are the eigenvalues of $\mathcal{H}^*
\mathcal{H}$ as given by Lemma \ref{lem:spectrum}.  Our proposed penalty serves the role of the
final two terms of the patch-based objective \eqref{eq:pbst_obj}.
The key difference is that the patch-based regularizer acts on the singular values of $W$, while the
proposed regularizer acts on the singular values of $\mathcal{H}$.

To satisfy (D\ref{d:dof}) we write the eigenvalues $\lambda_i$ in terms of the matrix $W$. Let $F
\in \Cbb^{N^2 \times N^2}$ denote the matrix that computes the $N \times N$ orthonormal 2D-DFT for a
vectorized signal, and let $\barQ \in \Cbb^{N^2 \times K^2}$ implement $N\times N$ 2D-DFT of a
zero-padded and vectorized $K \times K$ signal. The $i$-th column of $\bar{F} W^T$ contains the
(vectorized) 2D-DFT of the $i$-th filter. Then $\lambda_i = \sum_{j=1}^{N_c} \abs{\barQ
  W^T}_{i,j}^2$, and
\begin{equation}
  \label{eq:log_det}
  \log\det{\mathcal{H}^* \mathcal{H}} = \sum_{i=1}^{N_F^2} \log \left( \sum_{j=1}^{N_c} \abs{\barQ W^T}_{i,j} \right),
\end{equation}
where the absolute value and squaring operations are taken pointwise. We can reduce the
computational and memory burden of the algorithm by using smaller $N_F \times N_F$ DFTs, provided
that Corollary \ref{cor:linear_pr_condition} implies the corresponding linear convolution filter
bank is PR. We take $N_F = 4 K$, which is suitable for filter banks with condition
number less than $3$.

Similar to earlier work on analysis operator learning \cite{Yaghoobi2011, Yaghoobi2012,
  Yaghoobi2013}, we found that our tight frame penalty often resulted in transforms with many
uniformly zero filters. We prevent zero-norm filters by adding the log-barrier penalty $\sum_{j=1}^{N_c}
-\log\left( \norm{W_{j,:}}_2^2 \right)$.
The combined regularizer is written as
\begin{IEEEeqnarray}{rCl}
J_1(W) &= \frac{1}{2} \sum_{i=1}^{N_c}\norm{W_{i,:}}_2^2 &-
\sum_{i=1}^{N_F^2} {\log \left(\sum_{j=1}^{N_c} \abs{\barQ W^T}^2_{i, j} \right)} \nonumber \\
&&- \sum_{i=1}^{N_c} \log\left( \norm{W_{i,:}}_2^2 \right).
\label{eq:reg_no_coherence}
\end{IEEEeqnarray}
The following proposition (proved in Appendix \ref{app:norm_cond_proof}) indicates that $J_1$
promotes filter bank transforms that satisfy (D\ref{d:frame}).
\begin{proposition}
  \label{prop:reg_min}
  Let $W^*$ be a minimizer of $J_1$, and let $\mathcal{H}$ be the undecimated cyclic convolution filter
  bank generated by the rows of $W^*$. Then $\mathcal{H}$ implements a uniformly normalized tight
  frame expansion over the space of images, with filter squared norms equal to
  $2(1 + N_F^2 / N_c)$ and frame
  constant $2(1 + N_c / N_F^2)$.
\end{proposition}

Finally, we would like to discourage learning copies of the same filter. One option
is to apply a log barrier to the coherence between each pair of filters \cite{Hawe2013}:
\begin{equation}
\label{eq:coherence}
J_2(W) = \sum_{1 \leq i < j \leq N_c}
-\log\left( 1 -
\left(  \frac{\ip{W_{i,:}, W_{j,:}}} {\norm{W_{i,:}}_2 \norm{W_{j,:}}_2} \right)^2
\right).
\end{equation}
This penalty works well whenever the filters have small support $(K \leq 8)$. For larger
filters, we observed the algorithm often learned filters with disjoint support that are shifted versions of one
another. These filters do not cause a large value in \eref{eq:coherence}, yet provide no advantage
over a single filter. We modify our coherence penalty to discourage filters that differ by only a
linear phase term by applying \eref{eq:coherence} to the squared magnitude responses of our filters,
which we write as $J_2 \left(\left(\abs{\barQ W^T}^2\right)^T\right)$.

Our learning problem is written as
\begin{equation}
 \label{eq:fbst_obj_parameterized}
 \min_{W, Z} f(W, Z, x) + \mu J_1(W) + \lambda J_2\left(W\right) + \nu \psi(Z).
\end{equation}
The scalar $\mu > 0$ controls the strength of the UNTF penalty and should be large enough that the
learned filter bank is well conditioned, so that approximating the eigenvalues using $N_F \times
N_F$ DFTs remains valid. The non-negative scalar parameters $\lambda$ and $\nu$ control the emphasis
given to the coherence and sparsity penalties, respectively.

\subsection{Optimization Algorithm}
We use an alternating minimization algorithm to solve \eref{eq:fbst_obj_parameterized}. In the
\emph{sparse coding step}, we fix $W$ and solve \eref{eq:fbst_obj_parameterized} for $Z$.
In the second stage, called the \emph{transform update step}, we update our transform $W$
by minimizing \eref{eq:fbst_obj_parameterized} with fixed $Z$.  We use superscripts
to indicate iteration number, and we take $\mathcal{H}^{(k)}$ to mean the filter bank
generated using filters contained in the rows of $W^{(k)}$.

The sparse coding step reduces to
\begin{equation}
\label{eq:z_up}
Z^{(k+1)} = \argmin_Z \frac{1}{2}\norm{W^{(k)} X  - Z}_F^2 + \nu \psi(Z)
\end{equation}
with solution $Z^{(k+1)} = \prox{\psi}{\mathcal{H}^{(k)} x}{\nu}$.
Next, with $Z^{(k+1)}$ fixed, we update $W$ by solving
\begin{equation}
\label{eq:W_up}
W^{(k+1)} = \argmin_{W} f(W, Z^{(k+1)}, x) + \mu J_1(W) + \lambda
J_2\left(W\right).
\end{equation}
Unlike the square, patch-based case, we do not have a closed-form solution for \eref{eq:W_up} and we
must resort to an iterative method. The limited-memory BFGS (L-BFGS) algorithm works well in
practice.

The primary bottleneck in this approach is a line search step, which requires multiple
evaluations of the objective function \eqref{eq:fbst_obj_parameterized} with fixed $Z$. The cost of
this computation is dominated by evaluation of the sparsification residual $\norm{WX - Z}_F^2$. With
$X$ and $Z$ fixed we precompute and store the small matrices $G = X X^T$ and $Y = X Z^T$. Then the
sparsification residual is given by
\begin{equation}
  \norm{WX - Z}_F^2 = \trace{W^T W G} - 2 \trace{W Y} + \norm{Z}_F^2,
\end{equation}
We can use FFT-based convolution to compute $WX, X X^T$, and $X Z^T$, leading to an algorithm
whose computational cost per iteration scales gracefully with both image and patch size. Table
\ref{tab:cost} lists the per-iteration cost for each gradient and function evaluation.

\begin{algorithm}[t]
  \caption{Filter bank sparsifying transform learning}
  \label{alg:fbst}
  \begin{algorithmic}[1]
    \INPUT Image $x$, Initial transform $W^{(0)}$
    \State $Z^{(0)} \gets \prox{\psi}{\mathcal{H}^{(0)} x}{\nu}$
    \State $k \gets 0$
    \Repeat
    \State $
    W^{(k+1)} \gets \argmin_{W} f(W, x, Z^{(k)}) + \mu J_1(W) + \lambda J_2(W) $
    \State $Z^{(k+1)} \gets \prox{\psi}{\mathcal{H}^{(k+1)} x}{\nu}$
    \State $k \gets k + 1$
    \Until{Halting Condition}
  \end{algorithmic}
\end{algorithm}

{\renewcommand{\arraystretch}{1.5} 
\begin{table*}
  \centering
  \begin{tabular} {| c | c | c | }
    \hline
    Penalty & Evaluation & Gradient \\
    \hline
    $f(W, Z, x)$      & $\mathcal{O}(N_c N^2 \log(N))$  &
    $\mathcal{O}(N_c K^4 +  N_c N^2 \log(N))$   \\ \hline
    $J_1(W) $ & $\mathcal{O}\left(N_c N_F^2 \log(N_F) + N_F^2 K^2 + N_c K^2  \right)$ &
    $\mathcal{O}(K^4 N_c + (N_c + K^2) N_F^2 \log(N_F)$ \\ \hline
    $J_2(W) $ & $\mathcal{O}\left( N_c^2 K^2 \right)$ & $\mathcal{O}(N_c K^4)$  \\ \hline
  \end{tabular}
  \caption{Computational cost for function and gradient computation.}
  \label{tab:cost}
\end{table*}

\section{Application to Image Denoising}
\label{sec:denoising}

A preliminary version of our filter bank transform learning framework has been applied in an
``adaptive'' manner for magnetic resonance imaging\cite{Pfister2015}. Here, we restrict our
attention to image denoising in the ``universal'' paradigm. We wish to recover a clean image $x^*$
from a noisy copy, $y = x^* + e$, where $e \sim \mathcal{N}(0, \sigma^2 I_N)$. Let $\mathcal{H}$ be
a pre-trained sparsifying filter bank. We consider two methods for image denoising.

\subsection{Iterative denoising}
The first is to solve a regularized inverse problem using a transform sparsity penalty, written as
\begin{equation}
  \min_{x,z} \frac{\lambda_r}{2} \norm{y - x}_2^2 +
  \frac{1}{2} \norm{\mathcal{H} x - z}_2^2 + \nu \psi(z),
\end{equation}
where $\lambda_r>0$ controls the regularization strength. We solve this problem by alternating
minimization: we update $z$ for fixed $x$, and then update $x$ with $z$ fixed. This procedure is
summarized as Algorithm \ref{alg:iterative_denoise}. The eigenvalue decomposition of Lemma
\ref{lem:spectrum} provides an easy way to compute the necessary matrix inverse for cyclic
convolution filter banks. For linear convolution filter banks, we use Lemma \ref{lem:spectrum}
to implement a circulant preconditioner \cite{Chan1994}.

Algorithm \ref{alg:iterative_denoise} has three key parameters. The regularization parameter
$\lambda_r$ reflects the degree of confidence in the noisy observations $y$ and should be chosen
inversely proportional to the noise variance. The sparsity of the transform sparse code is controlled
by $\nu$.  The value of $\nu$ when denoising an image need not be the same as $\nu$
during the learning procedure and should be proportional to $\sigma$. The choice of both $\nu$ and
$\lambda_r$ depends on the final parameter: the number of iterations used during denoising.
Empirically, we've found that using $\mathrm{ceil}\left\{\sigma \cdot 255 / 10 \right\}$ iterations works
well; \ie, for $\sigma = 30/255$, we use $3$ iterations.

\begin{algorithm}
  \caption{Iterative denoising with filter bank transform}
  \label{alg:iterative_denoise}
  \begin{algorithmic}[1]
    \INPUT Noisy signal $y$, Learned filter bank transform $\mathcal{H}$
    \State $k \gets 0$
    \Repeat
    \State $z^{(k+1)} \gets \prox{\psi}{\mathcal{H} x^{(k)}}{\nu}$
    \State $x^{(k+1)} \gets (\mathcal{H}^* \mathcal{H} + \lambda_r I)^{-1} (\mathcal{H}^*
    z^{(k)} + \lambda_r y)$
    \State $k \gets k + 1$
    \Until{Halting Condition}
  \end{algorithmic}
\end{algorithm}

\subsection{Denoising by Transform-domain Thresholding}
We also consider a simpler algorithm, inspired by the transform domain denoising techniques of old.
We can form a denoised estimate by passing $y$ through the system in Figure
\ref{fig:filterbank_block_diagram}; that is, computing
\begin{equation}
  \label{eq:threshold_denoising}
  \mathcal{H}^\dagger \prox{\psi}{\mathcal{H} y}{\nu}.
\end{equation}
This approach simplifies denoising by eliminating two parameters from Algorithm
\ref{alg:iterative_denoise}: the number of iterations and $\lambda_r$.

Denoising in this manner is sensible because of the properties we have imposed on $\mathcal{H}$.
Noise in the signal will not be sparse in the transform domain and thus will be reduced by the nonlinearity.
The left-inverse is guaranteed to exist, and as $\mathcal{H}$ must be well-conditioned,
there will be little noise amplification due to the pseudoinverse. Finally, if $\mathcal{H}$ has low
coherence, the transformed noise $\mathcal{H}e$ will not be correlated across channels, suggesting
that a simple channelwise nonlinearity is sufficient. A multi-channel nonlinearity may be beneficial
if the transform is highly coherent.

\section{Experiments}
\label{sec:experiments}
We implemented a CPU and GPU version of our algorithms using NumPy 1.11.3 and SciPy
0.18.1.  Our GPU code interfaces with Python through PyCUDA and scikits.cuda, and implemented on an
NVidia Maxwell Titan X GPU.

We conducted training experiments using the five training images in Figure \ref{fig:training}. Each
image, in testing and training, was normalized to have unit $\ell_2$ norm. Unless otherwise
specified, our transforms were learned using $1000$ iterations of Algorithm \ref{alg:fbst} with
parameters $\mu=3.0$, $\lambda=7 \times 10^{-4}$, and $\nu = 5.5 \times 10^{-3}$. Sparsity was
promoted using an $\ell_0$ penalty, for which the prox operator corresponds to hard thresholding.

The initial transform $\mathcal{H}^{(0)}$ must be feasible, \ie left-invertible.
Random Gaussian and DCT initializations work well in practice.
We learned a $64$ channel filter bank with $8 \times 8$ filters using these initializations, and
learned filters are shown in Figure \ref{fig:dct_random}. The learned filters appear
similar and perform equally well in terms of sparsification and denoising.

More examples of learned filters and their magnitude frequency responses are shown in Figure
\ref{fig:filters}. We show a subset of channels from a filter bank consisting of $16 \times 16$
filters and $128$ channels. When viewed as a patch based transform, this filter bank would be
$2\times$ \emph{under}-complete. The ability to choose longer filters independent of the
number of channels is a key advantage of our framework over existing transform learning approaches.

\begin{figure}[t]
\begin{center}
  \includegraphics[scale=0.4]{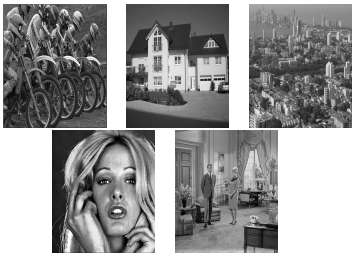}
\end{center}
\caption{Training images.}
\label{fig:training}
\end{figure}

\begin{figure}[ht]
  \begin{subfigure}{0.45\columnwidth}
    \includegraphics[scale=0.3]{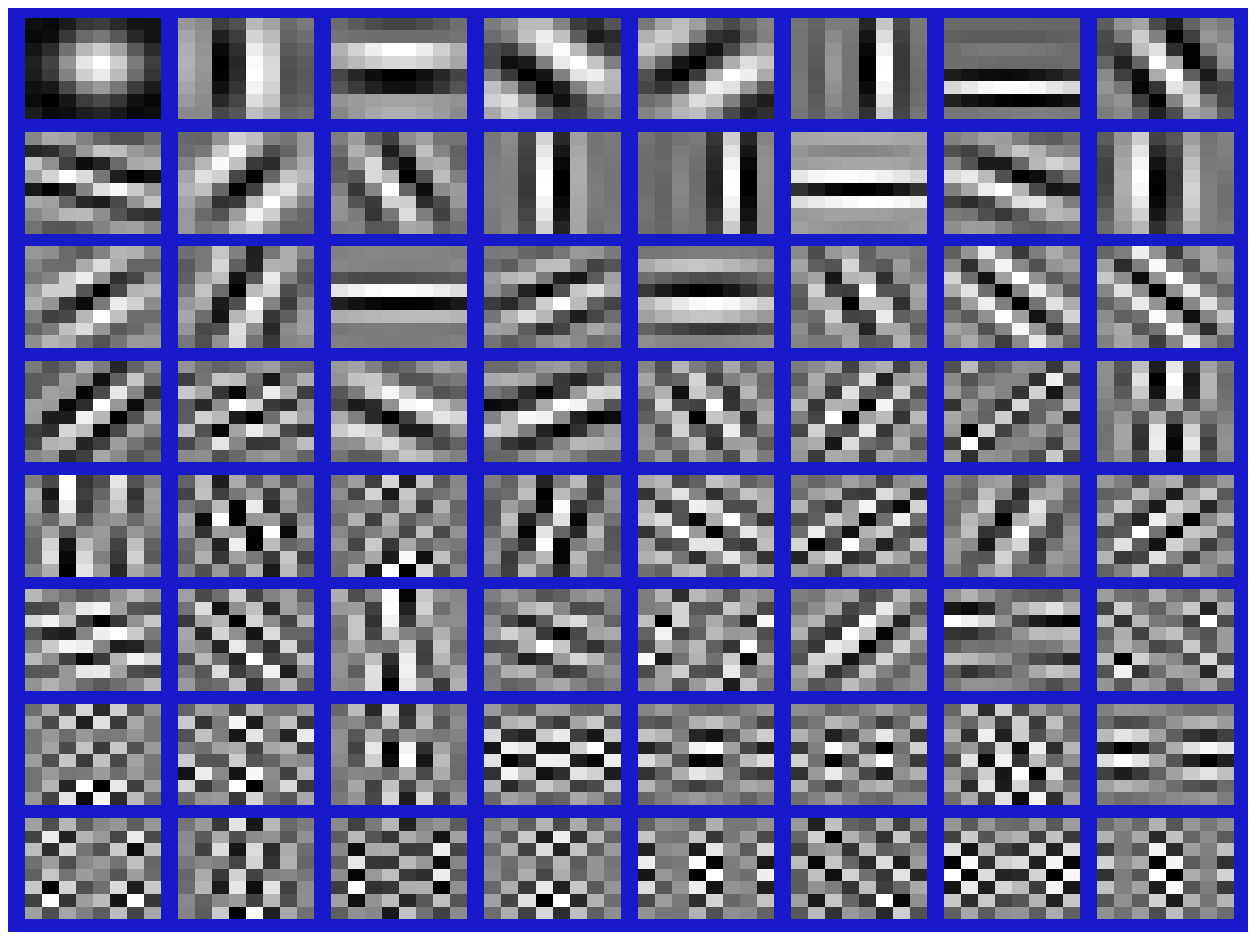}
    \caption{}
  \end{subfigure}
  \begin{subfigure}{0.45\columnwidth}
    \includegraphics[scale=0.3]{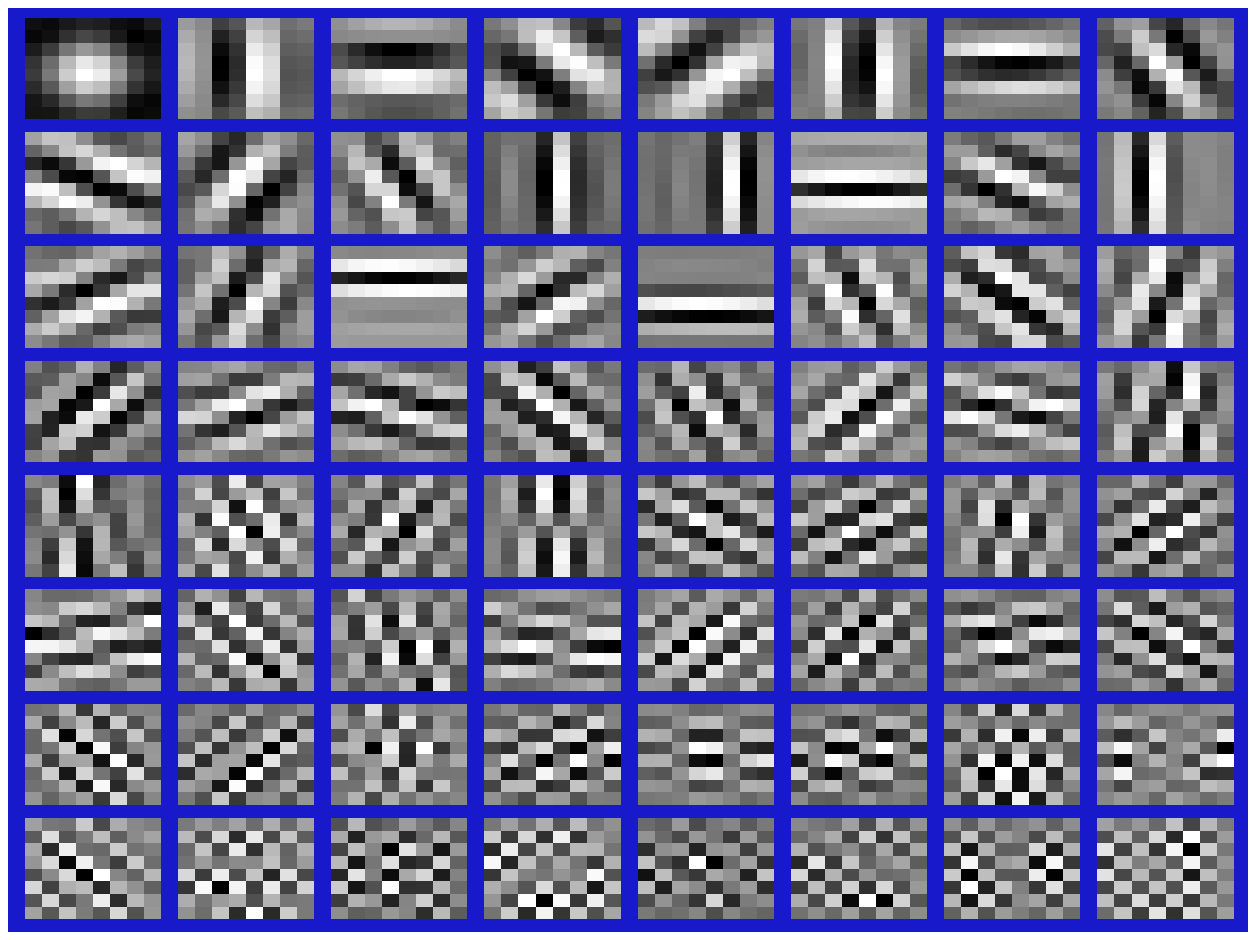}
    \caption{}
  \end{subfigure}
  \caption{
    Comparing initializations for $64$ channel filter bank with $8 \times 8$ filters.
    (a) Learned filters using random initialization;  (b) Filters learned with
    DCT initialization.}
  \label{fig:dct_random}
\end{figure}

\begin{figure}[t]
  \centering
  \begin{subfigure}{0.45\columnwidth}
    \centering
    \includegraphics[scale=0.35]{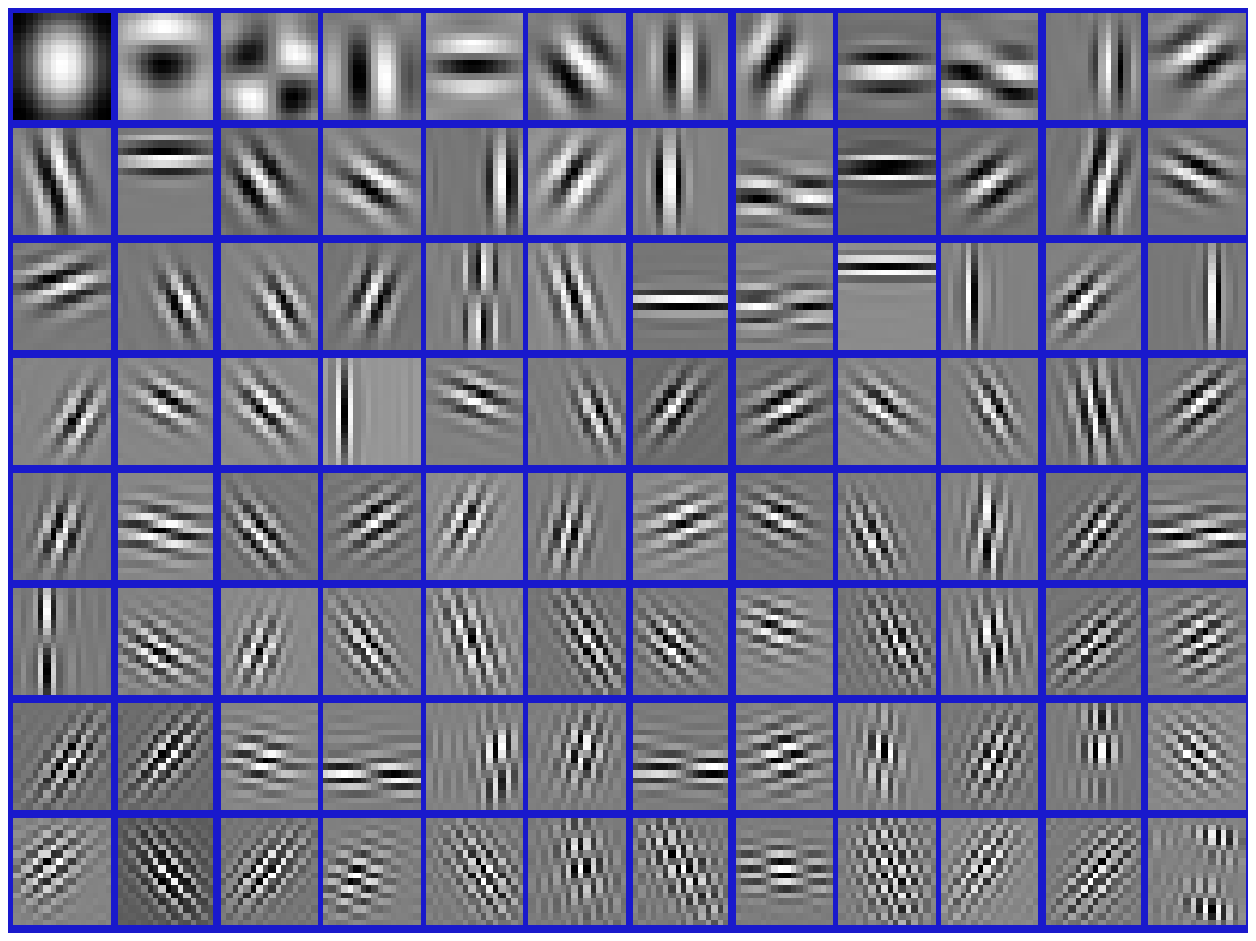}
    \caption{}
  \end{subfigure}
  \begin{subfigure}{0.45\columnwidth}
    \centering
    \includegraphics[scale=0.35]{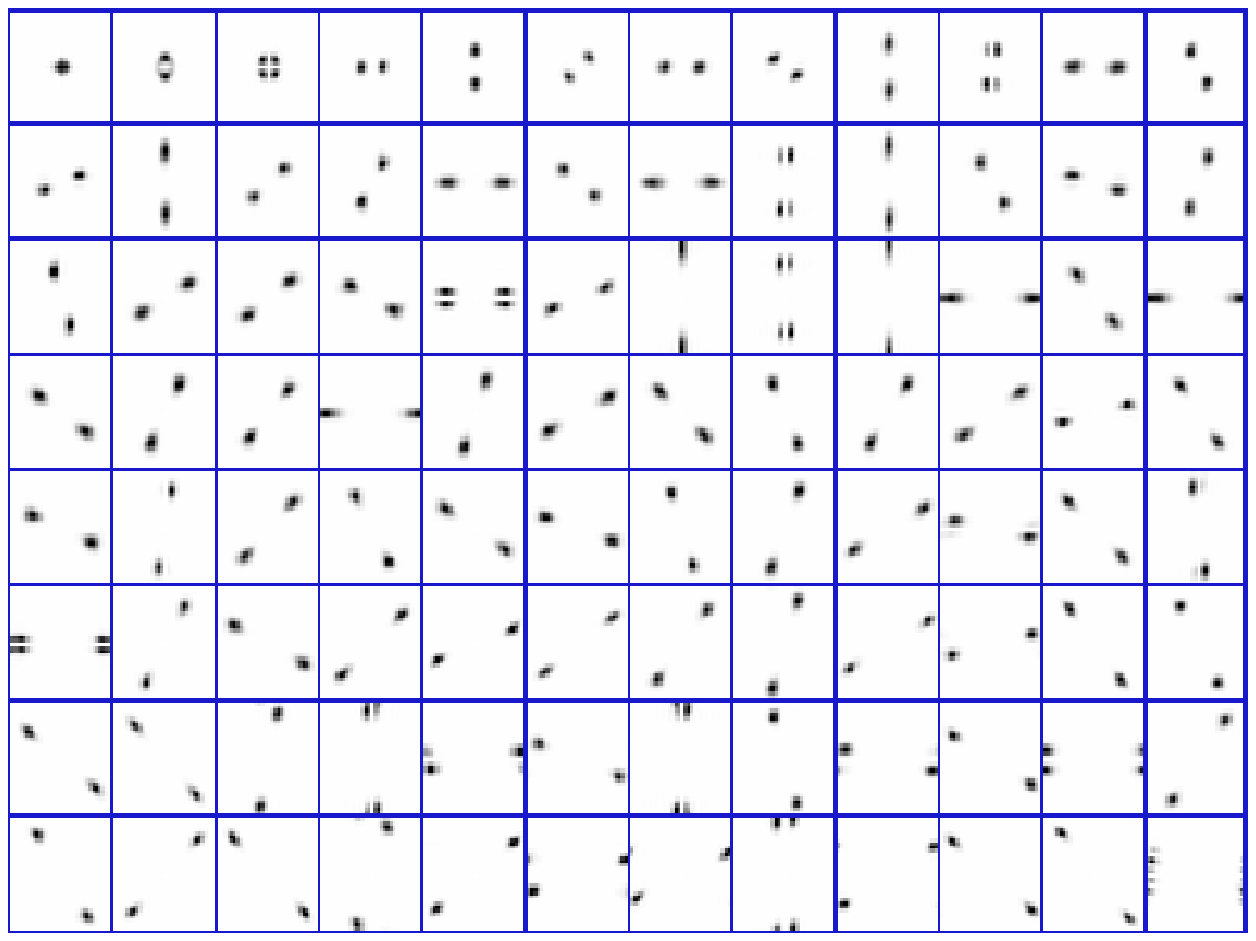}
    \caption{}
  \end{subfigure}
\caption{Examples of learned $16 \times 16$ filters. (a) Filter impulse responses; (b) Magnitude frequency responses.  The zero frequency (DC) is located at the center of each small box.}
\label{fig:filters}
\end{figure}

\begin{figure}[t]
  \centering
  \begin{subfigure}{0.45\columnwidth}
    \raisebox{2mm}{\includegraphics[scale=0.73]{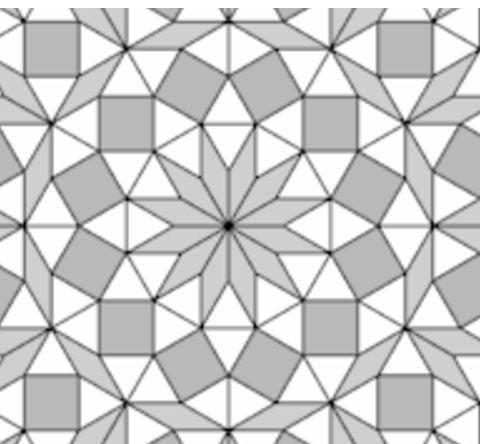}}
    \caption{}
  \end{subfigure}
  \begin{subfigure}{0.45\columnwidth}
    \includegraphics[scale=0.40]{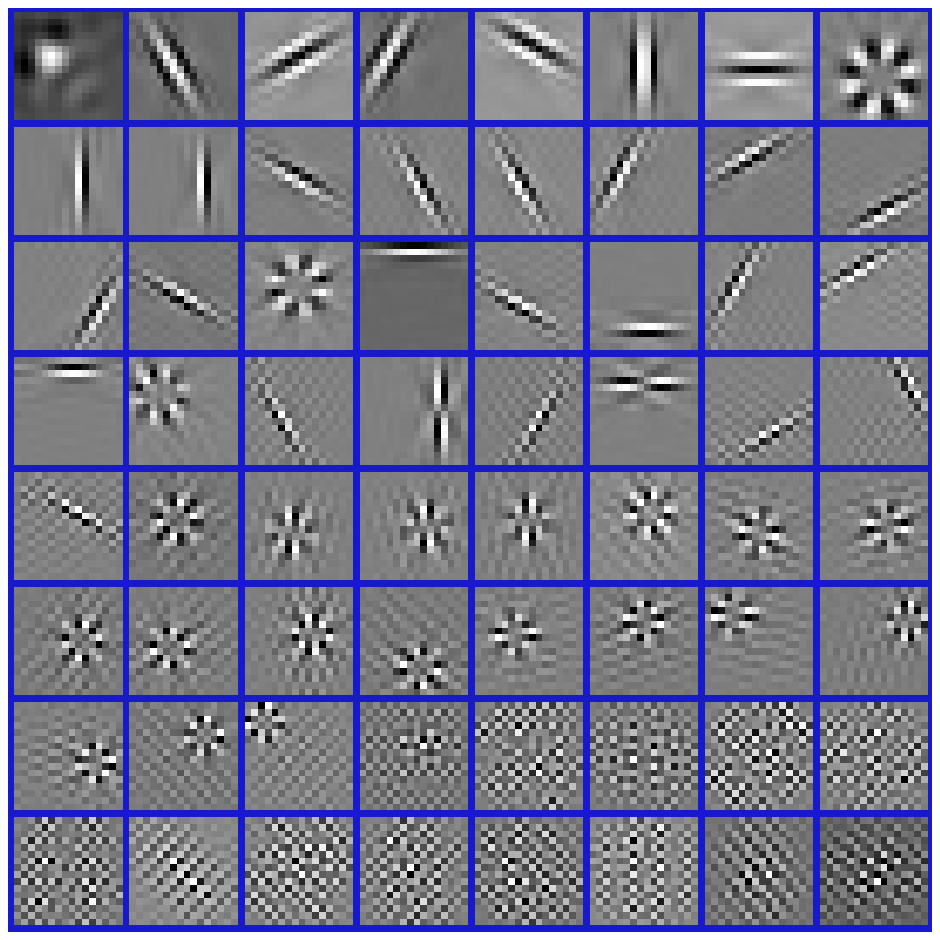}
    \caption{}
  \end{subfigure}
  \caption{Adaptivity of filters. (a) Training image; (b) Learned filters. }
  \label{fig:geometric_filters}
\end{figure}

\subsection{Image denoising}
We investigate the denoising performance of the filter bank sparsifying transforms as a function of
number of channels, $N_c$, and filter size, $K$, using our two algorithms. Our metric of interest is
the peak signal-to-noise ratio (PSNR) between the reconstructed image $x$ and the ground truth image
$x^*$, defined in decibels as $\textrm{PSNR} = 20 \log_{10} (N^2 / \norm{x - x^*}_2)$ . We evaluate
the denoising performance of our algorithm on the grayscale \texttt{barbara}, \texttt{man},
\texttt{peppers}, \texttt{baboon} and \texttt{boat} images.

We compare the denoising performance of our algorithm against two competing methods: 
EPLL \cite{Zoran2011} and BM3D \cite{Dabov2007}. We also evaluate image denoising
using filters learned with the square, patch-based transform learning algorithm
\cite{Ravishankar2013b}. We used $8 \times 8$ and $16\times 16$ image patches to learn a patch based
transform $W$. We used the rows of $W$ to generate an undecimated filter bank and used this filter
bank to denoise using Algorithm \ref{alg:iterative_denoise} and \eqref{eq:threshold_denoising}. The
filter bank implements cyclic convolution and image patches are extracted using periodic boundary
conditions. Observe for this fourth benchmark, that only the number of channels, regularizer and
learning algorithm differ from the more general filter bank transforms developed in Section
\ref{sec:learning}. All sparsifying transforms and EPLL were trained under the ``universal''
paradigm using all possible (maximally overlapping) patches extracted from the set of images shown in
Figure \ref{fig:training}.

We refer to transforms learned using our proposed Filter Bank Sparsifying Transform formulation as
``FBST'', and to the patch-based sparsifying transforms learned by solving \eqref{eq:pbst_obj} as
``PBST''. We evaluate our filter bank learning formulation using $64$, $128$, and $256$ channels
with $8 \times 8$ and $16 \times 16$ filters. During the denoising stage, we set $\nu = 10^{-4} \times
0.1 \sigma$ and $\lambda_r$ was adjusted for the particular noise level.

Mean reconstruction PSNR over the entire test set is shown in Table \ref{tab:psnr_mean}. Here,
FBST-64-8 indicates a $64$ channel filter bank with $8 \times 8$ filters where we denoise using
transform domain thresholding \eqref{eq:threshold_denoising}, while FBST-128-16-I indicates the use
of a $128$ channel filter bank with $16\times 16$ filters, where we denoise using the iterative
minimization Algorithm \ref{alg:iterative_denoise}. Per-image reconstruction PSNR is given in Table
\ref{tab:psnr}.

We see that the FBST performs on par with EPLL and slightly worse than BM3D. While the difference
between $128$ and $64$ filters appears to be minimal, in each case the transform learned using our
filter bank model outperforms the transform learned using the patch based model. On average, the $8
\times 8$ filters performed as well or better than the longer filters, but for certain images such
as $\texttt{barbara}$, see significant improvement when using long filters. Increasing the number of
channels beyond the filter size $K^2$ provides marginal improvement. For low noise,
denoising by transform-domain thresholding and iterative denoising using Algorithm
\ref{alg:iterative_denoise} perform equally well. As the noise level increases, the iterative
denoising algorithm outperforms the simpler thresholding scheme.

Using $1000$ iterations to learn a $196$ channel filter bank with $16 \times 16$ filters with
Algorithm \ref{alg:fbst} took roughly $5$ minutes on our GPU. In contrast, a GPU implementation of
the square, patch-based transform learning algorithm learns a $256 \times 256$ patch based transform
over the same data set in under one minute. This illustrates the efficiency of the closed-form
transform update step available in the patch-based case \cite{Ravishankar2013c}. Our slower learning
algorithm is offset by the ability to choose $N_c < K^2$, and this leads to faster application of
the learned transform to data. Comparing FBST-128-16-I and PBST-256-16-I, the image-based transform
outperforms the patch-based transform by up to $0.3$ dB despite containing half as many channels.

\subsection{Learning on subset of patches}
One advantage of  patch-based formulation is that the model can be trained using a large set images
by randomly selecting a few patches from each image. We can use the same approach when
learning a sparsifying filter bank:  the data matrix $X$ in \eqref{eq:se} is formed by extracting
and vectorizing patches from many images.  We can no longer view $WX$ as a convolution.

We learned a transform using $200,000$ randomly extracted patches from the 
training images in Figure \ref{fig:training}.  The learned transform performed nearly identically to
a transform learned using all patches from the training images.

\subsection{Image Adaptivity}
To test the influence of the training set, we learned a filter bank using $256^2$ patches of size $8
\times 8$ chosen at random from the $200$ training images in the BSDS300 training set
\cite{Martin2001}. The learned filter bank consists of Gabor-like filters,
much like filter banks learned from the images in
Figure~\ref{fig:training}. Gabor-like filters are naturally promoted by the
regularizers $J_1$ and $J_2$: their narrow support in the frequency domain leads to low coherence,
and their magnitude responses can tile frequency space leading to a well-conditioned transform.

We wonder if we have regularized our learning problem so strongly that the data
no longer plays a role. Fortunately, this is not so: Figure \ref{fig:geometric_filters}
illustrates a $64$ channel filter bank of $16 \times 16$ filters learned from a highly symmetric and 
geometric image. The learned filters include oriented edge detectors, as in the natural image case,
but also filters with a unique structure that sparsify the central region of the image.

\section{Remarks}
\label{sec:remarks}
Adaptive analysis/transform sparsity based image denoising algorithms can be coarsely divided into
two camps: supervised and unsupervised. In both cases, one learns a signal model by minimizing
an objective function.

In the supervised case, this minimization occurs over a set of training data.  In a
denoising application one typically corrupts a clean input with noise, passes it through the
denoiser, and uses the difference between the clean and denoised signal to adapt various parts
denoising algorithm:  the analysis operator, thresholding parameters, mixture weights,
the number of iterations, and so on.  It is not necessary to regularize the learning procedure to
preclude degenerate solutions, such as a transform of all zeros;  such a transform would not perform
well at the denoising task, and thus would not be learned by the algorithm
\cite{Roth2009,Peyre2011,Chen2014}.

In the unsupervised case, the objective function has two components. The first is a surrogate for the
true, but unavailable, metric of interest. In this paper, we use the combination of
sparsity and sparsification error to act as a surrogate for reconstruction PSNR. The second part of
the objective is a regularizer that prevents degenerate solutions, as discussed in 
\cref{sub:transform,sub:analysis}. Even in the ``universal'' case, our learning is
essentially unsupervised, as the learning process is not guided by the denoising PSNR.

The TNRD algorithm \cite{Chen2016} is a supervised approach that resembles iterative denoising using
Algorithm \ref{alg:iterative_denoise}, but where the filter coefficients, nonlinearities, and
regularization parameter are allowed to vary as a function of iteration. However, the TNRD approach
has no requirements that the filters form a well-conditioned frame or have low coherence; ``poor''
filters are simply discouraged by the supervised learning process. Denoising with the TNRD algorithm
outperforms the learned filter bank methods presented here.

One may ask if it is necessary that the learned transform be a frame. Indeed, the matrix to be
inverted when denoising using Algorithm \ref{alg:iterative_denoise} is full-rank even if the filter
bank itself is not perfect reconstruction. The proposed regularizer, while less restrictive
than previous transform learning regularizers, may still overly constrain the set of learnable
sparsifying transforms. However, our highly regularized approach has a benefit of its own. Whereas
the TNRD algorithm is trained on hundreds of images, and can take upwards of days to train, our
algorithm can be applied to a single image and requires only a few minutes. The tradeoff offered by
the TRND algorithm is acceptable for image denoising tasks, as massive sets of natural images are
publicly available for use in machine learning applications. However, such data sets may not be available for
new imaging modalities, in which case a tradeoff closer to that offered by our filter bank learning
algorithm may be preferred. Finding a balance between our highly regularized and unsupervised approach
and competing supervised learning methods is the subject of ongoing work.

\section{Conclusions}
\label{sec:conclusion}
We have developed an efficient method to learn sparsifying transforms that are structured as
undecimated, multi-dimensional perfect reconstruction filter banks. Unlike previous transform
learning algorithms, our approach can learn a transform with fewer rows than columns. We anticipate
this flexibility will be important when learning a transform for high dimensional data. Numerical
results show our filter bank sparsifying transforms outperform existing patch-based methods in image
denoising. Future work might fully embrace the filter bank perspective and learn filter bank
sparsifying transforms with various length filters and/or non-square impulse responses. It would
also be of interest to develop an accompanying algorithm to learn multi-rate filter bank sparsifying
transforms.

\appendices
\section{}
\label{app:conv_proof}
We explicitly show the link between filter banks and applying a sparsifying transform to a patch
matrix. We assume a 1D signal $x \in \Rbb^N$ to simplify notation. The extension to multiple
dimensions is tedious, but straightforward.

Let $W \in \Rbb^{N_c \times K}$ be a given transform, and let $w^i$ indicate the $i$-th row of this
matrix. Suppose we extract patches with a patch stride of $s$ and we assume $s$ evenly divides $N$.
The $j$-th column of the patch matrix $X \in \Rbb^{K \times M}$ is the vector
$[x_{sj + K - 1}, x_{sj + K - 2}, \hdots, x_{sj}]^T$.
The number of columns, $M$, depends on the boundary conditions used.
Linear and circular convolution are obtained by setting
$x_{i} = 0$ or $x_{i} = x_{N - i - 1}$, respectively, when $i < 0 $.  For cyclic convolution, we
have $M = N / s$.  The $i, j$-th element of the sparsified signal $WX$ is
\begin{align}
  [WX]_{i,j} &= \sum_{k=1}^{K} W_{i, k} X_{k, j} = \sum_{k=1}^K W_{i, k} x_{sj + K - 1 - i} \\
  &= (w^i * x)[sj + K - 1].
\end{align}
Thus the $i$-th row of $WX$ is the convolution between the filter with impulse response $w^i$ and
signal $x$, followed by
downsampling by a factor of $s$, and shifted by $K-1$. The filter bank has $N_c$ channels
with impulse responses given by the rows of $W$.  The shift of $K-1$ can be incorporated into the
definition of the patch extraction procedure.  For 1D signals, the ``first'' patch should be
$[x_{K-1}, \hdots x_{0}]^T$,  while for 2D signals, the lower-right pixel of the
  ``first'' patch is $x[0, 0]$.

\section{Proof of Proposition \ref{prop:reg_min}}
\label{app:norm_cond_proof} The function $J_1(W)$ in \eqref{eq:reg_no_coherence}  acts only on the magnitude
responses of the filters in $\mathcal{H}$. Let $V \triangleq \abs{\barQ W^T}^2 \in
\Rbb^{N_F^2 \times N_c}$. The sum of the $i$-th column of $V$ is equal to the norm
of the $i$-th filter and, by Lemma \ref{lem:spectrum}, the eigenvalues of
$\mathcal{H}^* \mathcal{H}$ are equal to the row sums of $V$. Thus, $V$ is
generated by a UNTF if and only if the row sums and column sums are
constant.

Let $V^*$ be a stationary point of $J_1$. For each $1
\leq r \leq N_F^2$ and $1 \leq s \leq N_c$, we have
\begin{align}
  \label{eq:opt_cond}
  \frac{\partial}{\partial V_{r, s}} J_1(V^*) =  \frac{1}{2} - \frac{1}{\sum_{j=1}^{N_c} V^*_{r, j}} - \frac{1}{\sum_{i=1}^{N_F^2} V^*_{i,s}} = 0.
\end{align}
Note that $J_1(V) = +\infty$ if either a row or column of $V$ is
identically zero, so $V^*$ is a minimizer only if there is at least one
non-zero in each row and column of $V^*$.
Subtracting $\frac{\partial}{\partial V_{r^\prime, s}} J_1(V^*)$ from
$\frac{\partial}{\partial V_{r, s}} J_1(V^*)$ yields
$\sum_{j=1}^{N_c} V^*_{r, j} = \sum_{j=1}^{N_c} V^*_{r^\prime, j} \triangleq a$.
Similarly, subtracting $\frac{\partial}{\partial V_{r, s}} J_1(V^*)$ from
$\frac{\partial}{\partial V_{r, s^\prime}} J_1(V^*)$ yields $\sum_{i=1}^{N_F^2}
V^*_{i, s} = \sum_{i=1}^{N_F^2} V^*_{i, s^\prime} \triangleq b$. As the row and
column sums are uniform for each $r$ and $s$, we conclude $V^*$ is a UNTF.
Next, we have
\begin{align}
  \sum_{i=1}^{N_F^2} \sum_{j=1}^{N_c} V^*_{i,j} &= \sum_{i=1}^{N_F^2} \left( \sum_{j=1}^{N_c} V^*_{i,j} \right) = N_F^2 a \\
  &= \sum_{j=1}^{N_C} \left( \sum_{i=1}^{N_F^2} V^*_{i,j} \right) = N_c b,
\end{align}
from which we conclude $b = \frac{N_F^2}{N_c} a$.
Substituting into \eref{eq:opt_cond}, we find
\begin{align}
  a = 2\left(1 + \frac{N_c}{N_F^2}\right), \quad \quad b = 2 \left(\frac{N_F^2}{N_c} + 1\right),
\end{align}
and this completes the proof.

\setlength{\tabcolsep}{.2em}
\renewcommand{\arraystretch}{1}
\begin{table*}[ht]
  \caption{Mean reconstruction PSNR for the test images \texttt{barbara}, \texttt{man},
    \texttt{peppers}, \texttt{baboon} and \texttt{boat}, averaged over $10$ noise realizations.
    FBST-128-16 indicates a filter bank sparsifying transform with $128$ channels and $16 \times 16$
    filters and denoised according to \cref{eq:threshold_denoising}. The -I suffix indicates
    denoising with the iterative \cref{alg:iterative_denoise}. \vspace{-2mm}}
  \centering
  \begin{tabular}{| c || c | c | c | c | c | c | c | c | c | c | c |}
    \hline
    \input{results_table2.tex}
  \end{tabular}
  \label{tab:psnr_mean}
\end{table*}

\setlength{\tabcolsep}{.5em}
  \renewcommand{\arraystretch}{1.1}
\begin{table*}[ht]
  \caption{Per-image reconstruction PSNR averaged over $10$ noise realizations.
    \vspace{-2mm}}
  \centering
  \begin{tabular}{|c||c|c|c||c|c|c||c|c|c||c|c|c||c|c|c|}
    \hline
    $\sigma$ & $10$ & $20$ & $30$
    & $10$ & $20$ & $30$
    & $10$ & $20$ & $30$
    & $10$ & $20$ & $30$
    & $10$ & $20$ & $30$ \\ \hline
    Input PSNR & $28.13$ & $22.11$ & $18.59$
               & $28.13$ & $22.11$ & $18.59$
               & $28.13$ & $22.11$ & $18.59$
               & $28.13$ & $22.11$ & $18.59$
               & $28.13$ & $22.11$ & $18.59$ \\ \hline \hline

    \input{results_table.tex}
  \end{tabular}
  \label{tab:psnr}
  \end{table*}

\label{sec:refs}
\bibliographystyle{myIEEEtran}
\bibliography{IEEEabrv,/home/luke/research/bib/jabref}



\end{document}

%% file: figs/patch_fig.tex
\newcounter{row}
\newcounter{col}

\newcommand\setrow[4]{
  \setcounter{col}{1}
  \foreach \n in {#1, #2, #3, #4} {
    \edef\x{\value{col}*0.8 - 0.4}
    \edef\y{2.8 - \value{row}*0.8}
    \node[anchor=center] at (\x, \y) {\Large \n};
    \stepcounter{col}
  }
  \stepcounter{row}
}
\newcommand\setcolvec[3]{
  \setcounter{col}{1}
  \foreach \n in {#1, #2, #3} {
    \edef\x{\value{col}*0.8 - 0.4}
    \edef\y{3.2 - \value{row}*0.8}
    \node[anchor=center] at (\x, \y) {\Large \n};
    \stepcounter{col}
  }
}

\begin{tikzpicture}
  \edef\Nc{4}
  \edef\Nr{3}
  \edef\gs{0.8}
  \edef\gx{\Nc*\gs}
  \edef\gy{\Nr*\gs}
  \draw[step=\gs] (0,0) grid (\gx,\gy);
  \setcounter{row}{1}
  \setrow {1}{2}{3}{4}
  \setrow {5}{6}{7}{8}
  \setrow {9}{10}{11}{12}

    \node[anchor=center, font=\huge] (arrow) at (\gx+0.5, 0.5*\gy){$\to$};

    \fill[red, opacity=.2, step=\gs] (0.0, \Nr*\gs - 2*\gs) rectangle (2*\gs, \Nr*\gs);

    \edef\x{1 cm + \gx cm}
    \edef\y{-0.4cm}
    \draw[step=\gs,xshift=\x,yshift=\y](0,0) grid (4*\gs, 4*\gs);  

    \fill[red,  opacity=.2, step=0.8cm,xshift=\x,yshift=\y] (0.0, 0.0) rectangle (\gs, 4*\gs);

    \edef\gs2{0.5*\gs cm}
    \node[anchor=center,font=\Large] (X_bottom_left) at (\x+\gs2,\y+\gs2){$1$};
    \node[anchor=center,font=\Large] at (\x+\gs2,\y+3*\gs2){$2$};
    \node[anchor=center,font=\Large] at (\x+\gs2,\y+5*\gs2){$5$};
    \node[anchor=center,font=\Large] at (\x+\gs2,\y+7*\gs2){$6$};
    \node[draw=none,fill=none,font=\large] at (\x+0.39cm, \y+3.5cm){$\Rj{1} x$};

    \node[anchor=center,font=\Large] at (\x+3*\gs2,\y+1*\gs2){$3$};
    \node[anchor=center,font=\Large] at (\x+3*\gs2,\y+3*\gs2){$4$};
    \node[anchor=center,font=\Large] at (\x+3*\gs2,\y+5*\gs2){$7$};
    \node[anchor=center,font=\Large] at (\x+3*\gs2,\y+7*\gs2){$8$};
    \node[draw=none,fill=none,font=\large] at (\x+1.21cm, \y+3.5cm){$\Rj{2} x$};

    \node[anchor=center,font=\Large] at (\x+5*\gs2,\y+1*\gs2){$9$};
    \node[anchor=center,font=\Large] at (\x+5*\gs2,\y+3*\gs2){$10$};
    \node[anchor=center,font=\Large] at (\x+5*\gs2,\y+5*\gs2){$1$};
    \node[anchor=center,font=\Large] at (\x+5*\gs2,\y+7*\gs2){$2$};
    \node[draw=none,fill=none,font=\large] at (\x+2.0cm, \y+3.5cm){$\Rj{3} x$};

    \node[anchor=center,font=\Large] (X_bottom_right) at (\x+7*\gs2,\y+1*\gs2){$11$};
    \node[anchor=center,font=\Large] at (\x+7*\gs2,\y+3*\gs2){$12$};
    \node[anchor=center,font=\Large] at (\x+7*\gs2,\y+5*\gs2){$3$};
    \node[anchor=center,font=\Large] at (\x+7*\gs2,\y+7*\gs2){$4$};
    \node[draw=none,fill=none,font=\large] at (\x+2.8cm, \y+3.5cm){$\Rj{4} x$};

    \draw [decorate,decoration={brace,mirror,amplitude=7}] ([yshift=-5mm] X_bottom_left.west) --node[below=2mm,font=\huge]{$X$} ([yshift=-5mm] X_bottom_right.east);



    \node[draw=none,fill=none,font=\huge] at   (1.6cm, \y-0.6cm){$x$};

  \end{tikzpicture}


%% file: figs/filter_from_patch.tex
\begin{tikzpicture}
    \edef\x{0cm}
    \edef\y{-0.5cm}
    \draw[step=0.8cm,xshift=\x,yshift=\y](0,0) grid (0.8,3.2);  
    \node[anchor=center] (X_bottom_left) at (\x+0.4cm,\y+0.4cm){\Large$1$};
    \node[anchor=center] at (\x+0.4cm,\y+1.2cm){\Large$2$};
    \node[anchor=center] at (\x+0.4cm,\y+2.0cm){\Large$3$};
    \node[anchor=center] at (\x+0.4cm,\y+2.8cm){\Large$4$};
    \node[draw=none,fill=none,font=\huge] at (\x+0.39cm, \y-0.6cm){$w^i$};
    \node[anchor=center, font=\huge] (arrow2) at (1.4, 1.1cm){$\to$};
    \node[anchor=center, font=\Large] at ([yshift=3mm] arrow2.north) {$\Rj{1}^*$};

    \edef\x{2.0cm}
    \edef\y{0.4cm}
    \draw[step=0.8cm,xshift=\x] (0,0) grid (2.4,2.4);
    \node[anchor=center] at (\x+0.4cm, \y+1.6cm){\Large $1$};
    \node[anchor=center] at (\x+0.4cm, \y+0.8cm){\Large $3$};
    \node[anchor=center] at (\x+0.4cm, \y+0.0cm){\Large $0$};
    \node[anchor=center] at (\x+1.2cm, \y+1.6cm){\Large $2$};
    \node[anchor=center] at (\x+1.2cm, \y+0.8cm){\Large $4$};
    \node[anchor=center] at (\x+1.2cm, \y+0.0cm){\Large $0$};
    \node[anchor=center] at (\x+2.0cm, \y+1.6cm){\Large $0$};
    \node[anchor=center] at (\x+2.0cm, \y+0.8cm){\Large $0$};
    \node[anchor=center] at (\x+2.0cm, \y+0.0cm){\Large $0$};
    \node[draw=none,fill=none,font=\huge] at (\x+1.2cm, \y-1.5cm){$h_i$};
    
  \end{tikzpicture}


%% file: figs/filterbank_analysis_only.tex

\begin{tikzpicture}[decoration={brace,mirror, amplitude=7}]
\makeatletter
\DeclareRobustCommand{\rvdots}{%
  \vbox{
    \baselineskip4\p@\lineskiplimit\z@
    \kern-\p@
    \hbox{.}\hbox{.}\hbox{.}
  }}

\makeatother
\matrix (m1) [row sep=1.5mm, column sep=1.8mm]
    {
      \node[coordinate]         (m00) {}; &
      \node[coordinate]         (m01) {}; &
      \node[dspnodeopen]        (m02) {}; &
      \node[coordinate]         (m03) {}; &
      \node[dspsquare]          (m04) {$\Rj{1}^* w^{1}$}; &
      \node[coordinate]         (m05) {}; &
      \node[dspsquare]          (m06) {$\downsamplertext{s I}$}; &
      \node[coordinate]         (m07) {}; &
      \node[dspnodeopen]        (m08) {$\mathcal{H}_{1} x$};  &
\\

      \node[dspnodeopen]        (m10) {$x$};  &
      \node[coordinate]         (m11) {}; &
      \node[dspnodeopen]        (m12) {}; &
      \node[coordinate]         (m13) {}; &
      \node[dspsquare]          (m14) {$\Rj{1}^* w^{2}$}; &
      \node[coordinate]         (m15) {}; &
      \node[dspsquare]          (m16) {\ $\downsamplertext{s I}$ \ }; &
      \node[coordinate]         (m17) {}; &
      \node[dspnodeopen]        (m18) {$\mathcal{H}_2 x$};  &
\\
      \node[coordinate]         (m20) {}; &
      \node[coordinate]         (m21) {}; &
      \node                     (m22) {\rvdots}; &
      \node[coordinate]         (m23) {}; &
      \node                     (m24) {\rvdots}; &
      \node[coordinate]         (m25) {}; &
      \node                     (m26) {\rvdots}; &
      \node[coordinate]         (m27) {}; &
      \node                     (m28) {\rvdots};  &
\\

      \node[coordinate]         (m30) {}; &
      \node[coordinate]         (m31) {}; &
      \node[dspnodeopen]        (m32) {}; &
      \node[coordinate]         (m33) {}; &
      \node[dspsquare]          (m34) {$\Rj{1}^* w^{N_c}$}; &
      \node[coordinate]         (m35) {}; &
      \node[dspsquare]          (m36) {\ \ $\downsamplertext{s I}$  \ }; &
      \node[coordinate]         (m37) {}; &
      \node[dspnodeopen]        (m38) {$\mathcal{H}_{N_c} x$};  &
      \\
    };

\foreach \i [evaluate = \i as \j using int(\i+2)] in {2,4,...,6}
		\draw[dspconn] (m0\i) -- (m0\j);
\foreach \i [evaluate = \i as \j using int(\i+2)] in {2,4,...,6}
		\draw[dspconn] (m1\i) -- (m1\j);
\foreach \i [evaluate = \i as \j using int(\i+2)] in {2,4,...,6}
		\draw[dspconn] (m3\i) -- (m3\j);

	\draw[dspconn] (m10) -- (m12);
	\draw[dspconn] (m12) -- (m02);
	\draw[dspconn] (m12) -- (m22);
	\draw[dspconn] (m22) -- (m32);

    \draw [decorate] ([yshift=-5mm]m32.west) --node[below=2mm]{$\mathcal{H}$} ([yshift=-5mm] m37.south);

\end{tikzpicture}


%% file: figs/filterbank_fig.tex

\begin{tikzpicture}[decoration={brace,mirror, amplitude=7}]
\makeatletter
\DeclareRobustCommand{\rvdots}{%
  \vbox{
    \baselineskip4\p@\lineskiplimit\z@
    \kern-\p@
    \hbox{.}\hbox{.}\hbox{.}
  }}

\makeatother
\matrix (m1) [row sep=1.5mm, column sep=1.8mm]
    {
      \node[coordinate]         (m00) {}; &
      \node[coordinate]         (m01) {}; &
      \node[dspnodeopen]        (m02) {}; &
      \node[coordinate]         (m03) {}; &
      \node[dspsquare]          (m04) {$h_1$}; &
      \node[coordinate]         (m05) {}; &
      \node[dspsquare]          (m06) {\ $\prox{\psi}{\cdot}{\nu}$ \ }; &
      \node[coordinate]         (m07) {}; &
      \node[dspnodeopen]        (m08) {$Z_{1,:}$};  &
      \node[coordinate]         (m09) {}; &
      \node[dspsquare]          (m010) {$\bar{h}_1$}; &
      \node[coordinate]         (m011) {}; &
      \node[dspnodeopen]        (m012) {}; &
\\

      \node[dspnodeopen]        (m10) {$x$};  &
      \node[coordinate]         (m11) {}; &
      \node[dspnodeopen]        (m12) {}; &
      \node[coordinate]         (m13) {}; &
      \node[dspsquare]          (m14) {$h_2$}; &
      \node[coordinate]         (m15) {}; &
      \node[dspsquare]          (m16) {\ $\prox{\psi}{\cdot}{\nu}$ \ }; &
      \node[coordinate]         (m17) {}; &
      \node[dspnodeopen]        (m18) {$Z_{2,:}$};  &
      \node[coordinate]         (m19) {}; &
      \node[dspsquare]          (m110) {$\bar{h}_2$}; &
      \node[coordinate]         (m111) {}; &
      \node[dspnodeopen]        (m112) {}; &
      \node[coordinate]         (m113) {}; &
      \node[dspsquare]          (m114) {$g$}; &
      \node[coordinate]         (m115) {}; &
      \node[dspnodeopen]        (m116) {$\hat{x}$}; &
\\
      \node[coordinate]         (m20) {}; &
      \node[coordinate]         (m21) {}; &
      \node                     (m22) {\rvdots}; &
      \node[coordinate]         (m23) {}; &
      \node                     (m24) {\rvdots}; &
      \node[coordinate]         (m25) {}; &
      \node                     (m26) {\rvdots}; &
      \node[coordinate]         (m27) {}; &
      \node                     (m28) {\rvdots};  &
      \node[coordinate]         (m29) {}; &
      \node                     (m210) {\rvdots};  &
      \node[coordinate]         (m211) {}; &
      \node                     (m212) {\rvdots};  &
\\

      \node[coordinate]         (m30) {}; &
      \node[coordinate]         (m31) {}; &
      \node[dspnodeopen]        (m32) {}; &
      \node[coordinate]         (m33) {}; &
      \node[dspsquare]          (m34) {$h_{N_c}$}; &
      \node[coordinate]         (m35) {}; &
      \node[dspsquare]          (m36) {\ $\prox{\psi}{\cdot}{\nu}$ \ }; &
      \node[coordinate]         (m37) {}; &
      \node[dspnodeopen]        (m38) {$Z_{N_c,:}$};  &
      \node[coordinate]         (m39) {}; &
      \node[dspsquare]          (m310) {$\bar{h}_{N_c}$}; &
      \node[coordinate]         (m311) {}; &
      \node[dspnodeopen]        (m312) {}; &
      \node[coordinate]         (m313) {}; &
      \node[coordinate]         (m314) {}; &
      \node[coordinate]         (m315) {}; &
      \node[coordinate]         (m316) {}; &
      \\
    };

\foreach \i [evaluate = \i as \j using int(\i+2)] in {2,4,...,10}
		\draw[dspconn] (m0\i) -- (m0\j);
\foreach \i [evaluate = \i as \j using int(\i+2)] in {2,4,...,14}
		\draw[dspconn] (m1\i) -- (m1\j);
\foreach \i [evaluate = \i as \j using int(\i+2)] in {2,4,...,10}
		\draw[dspconn] (m3\i) -- (m3\j);

	\draw[dspconn] (m10) -- (m12);
	\draw[dspconn] (m12) -- (m02);
	\draw[dspconn] (m12) -- (m22);
	\draw[dspconn] (m22) -- (m32);

	\draw[dspconn] (m012) -- (m112);
	\draw[dspconn] (m212) -- (m112);
	\draw[dspconn] (m312) -- (m212);

    \draw [decorate] ([yshift=-5mm]m32.west) --node[below=2mm]{$\mathcal{H}$} ([yshift=-5mm] m35.south);
    \draw [decorate] ([yshift=-5mm]m39.west) --node[below=2mm]{$\mathcal{H^\dagger}$} ([yshift=-5mm] m316.east);

\end{tikzpicture}


%% file: results_table2.tex
$\sigma$       & BM3D  &  EPLL  & FBST-64-8   & FBST-128-8   & FBST-196-8   & FBST-64-16   & FBST-128-16   & FBST-196-16   & PBST-64-8   & PBST-256-16\\ \hline
10             & 33.60 &  33.26 &  33.32      &  33.34       &  33.32       &  33.28       &  33.34        &  33.35        &  33.13      &  33.28 \\ \hline
 20            & 30.42 &  30.01 &  29.79      &  29.83       &  29.81       &  29.78       &  29.92        &  29.95        &  29.55      &  29.71 \\ \hline
 30            & 28.61 &  28.15 &  27.77      &  27.83       &  27.80       &  27.74       &  27.96        &  28.03        &  27.55      &  27.65 \\ \hline
\hline$\sigma$ & BM3D  &  EPLL  & FBST-64-8-I & FBST-128-8-I & FBST-196-8-I & FBST-64-16-I & FBST-128-16-I & FBST-196-16-I & PBST-64-8-I & PBST-256-16-I\\ \hline
10             & 33.60 &  33.26 &  33.41      &  33.40       &  33.37       &  33.41       &  33.44        &  33.45        &  33.13      &  33.28 \\ \hline
 20            & 30.42 &  30.01 &  30.03      &  30.02       &  29.97       &  30.05       &  30.13        &  30.16        &  29.74      &  29.85 \\ \hline
 30            & 28.61 &  28.15 &  28.10      &  28.10       &  28.01       &  28.13       &  28.27        &  28.31        &  27.82      &  27.89 \\ \hline

%% file: results_table.tex
\textbf{Method} &\multicolumn{3}{c||}  {\textbf{ baboon }} &\multicolumn{3}{c||}  {\textbf{ barbara }} &\multicolumn{3}{c||}  {\textbf{ boat }} &\multicolumn{3}{c||}  {\textbf{ man }} &\multicolumn{3}{c||}  {\textbf{ peppers }} \\ \hline 
BM3D&  30.47 &  26.45 &  24.40 &  34.95 &  31.74 &  29.78 &  33.90 &  30.84 &  29.04 &  33.94 &  30.56 &  28.81 &  34.75 &  32.51 &  31.01 \\ \hline 
EPLL&  30.49 &  26.56 &  24.54 &  33.63 &  29.81 &  27.61 &  33.64 &  30.68 &  28.90 &  33.97 &  30.62 &  28.82 &  34.58 &  32.36 &  30.88 \\ \hline 
FBST-64-8  &  30.33  &  26.12  &  23.99  &  34.36  &  30.52  &  28.20  &  33.66  &  30.30  &  28.33  &  33.70  &  30.13  &  28.19  &  34.53  &  31.90  &  30.12  \\ \hline 
PBST-64-8  &  30.27  &  25.99  &  23.83  &  34.00  &  30.02  &  27.70  &  33.54  &  30.18  &  28.25  &  33.44  &  29.89  &  28.03  &  34.40  &  31.67  &  29.92  \\ \hline 
FBST-128-16  &  30.30  &  26.13  &  24.03  &  34.60  &  30.99  &  28.80  &  33.65  &  30.34  &  28.43  &  33.60  &  30.10  &  28.20  &  34.56  &  32.01  &  30.34  \\ \hline 
PBST-256-16  &  30.37  &  26.12  &  23.94  &  34.21  &  30.32  &  28.01  &  33.63  &  30.21  &  28.23  &  33.63  &  29.97  &  27.97  &  34.57  &  31.90  &  30.12  \\ \hline 
FBST-64-8-I  &  30.41  &  26.34  &  24.24  &  34.42  &  30.76  &  28.56  &  33.78  &  30.52  &  28.64  &  33.79  &  30.30  &  28.42  &  34.65  &  32.22  &  30.62  \\ \hline 
PBST-64-8-I  &  30.19  &  26.09  &  23.99  &  34.15  &  30.33  &  28.10  &  33.44  &  30.24  &  28.40  &  33.56  &  30.08  &  28.25  &  34.34  &  31.94  &  30.38  \\ \hline 
FBST-128-16-I  &  30.39  &  26.34  &  24.31  &  34.72  &  31.26  &  29.21  &  33.75  &  30.55  &  28.70  &  33.66  &  30.23  &  28.39  &  34.66  &  32.28  &  30.73  \\ \hline 
PBST-256-16-I  &  30.32  &  26.15  &  24.05  &  34.25  &  30.55  &  28.29  &  33.61  &  30.22  &  28.36  &  33.66  &  30.20  &  28.24  &  34.57  &  32.10  &  30.50  \\ \hline 